\documentclass{article}

    \usepackage[preprint]{neurips_2022}

\usepackage[utf8]{inputenc} 
\usepackage[T1]{fontenc}    
\usepackage{hyperref}       
\usepackage{url}            
\usepackage{booktabs}       
\usepackage{amsfonts}       
\usepackage{nicefrac}       
\usepackage{microtype}      
\usepackage{xcolor}         

\usepackage{hhline}
\usepackage{derivative}
\usepackage{multirow}
\usepackage[ruled,vlined]{algorithm2e}
\usepackage{algorithmic}

\usepackage[thinc]{esdiff}

\newcommand{\cG}{{\mathcal{G}}}
\newcommand{\cV}{{\mathcal{V}}}
\newcommand{\cD}{{\mathcal{D}}}

\newcommand{\cP}{{\mathcal{P}}}
\newcommand{\cL}{{\mathcal{L}}}

\newcommand{\cF}{{\mathcal{F}}}
\newcommand{\cN}{{\mathcal{N}}}

\newcommand{\cR}{{\mathcal{R}}}
\newcommand{\cT}{{\mathcal{T}}}

\newcommand{\cH}{{\mathcal{H}}}

\newcommand{\cO}{{\mathcal{O}}}

\newcommand{\bx}{\textbf{x}}
\newcommand{\by}{\textbf{y}}

\newcommand{\bc}{\textbf{c}}
\newcommand{\br}{\textbf{r}}

\newcommand{\bp}{\textbf{p}}
\newcommand{\ba}{\textbf{a}}
\newcommand{\bv}{\textbf{v}}
\newcommand{\bz}{\textbf{z}}

\newcommand{\bl}{\textbf{l}}

\newcommand{\bbR}{\mathbb{R}}
\newcommand{\bbE}{\mathbb{E}}

\usepackage{amsmath}
\usepackage{amssymb}
\usepackage{mathtools}
\usepackage{amsthm}

\usepackage{mathtools}

\DeclarePairedDelimiter\floor{\lfloor}{\rfloor}
\newcommand\myeq{\stackrel{\mathclap{\tiny\mbox{def}}}{=}}

\newtheorem*{theorem*}{Theorem}
\newtheorem*{lemma*}{Lemma}

\usepackage{mathtools}

\newif\ifnotes\notestrue
\def\htien#1{}

\catcode`@=11
\def\caseswithdelim#1#2{\left#1\,\vcenter{\normalbaselines\m@th
  \ialign{\strut$##\hfil$&\quad##\hfil\crcr#2\crcr}}\right.}
\catcode`@=12
\usepackage{hyperref}

\usepackage[capitalize,noabbrev]{cleveref}

\theoremstyle{plain}
\newtheorem{theorem}{Theorem}

\newtheorem{lemma}{Lemma}

\theoremstyle{definition}

\theoremstyle{remark}

\usepackage[textsize=tiny]{todonotes}

\title{Scalable Distributional Robustness in a Class of Non Convex Optimization with Guarantees}

\author{Avinandan Bose \\Singapore Management University\\avinandanb@smu.edu.sg \And Arunesh Sinha \\Singapore Management University \\aruneshsinha@gmail.com \And Tien Mai \\Singapore Management University \\atmai@smu.edu.sg
}

\begin{document}

\maketitle

\begin{abstract}
    Distributionally robust optimization (DRO) has shown lot of promise in providing robustness in learning as well as sample based optimization problems. We endeavor to provide DRO solutions for a class of sum of fractionals, non-convex optimization which is used for decision making in prominent areas such as facility location and security games. In contrast to previous work, we find it more tractable to optimize the equivalent variance regularized form of DRO rather than the minimax form. We transform the variance regularized form to a mixed-integer second order cone program (MISOCP), which, while \emph{guaranteeing near global optimality}, does not scale enough to solve problems with real world data-sets. We further propose two abstraction approaches based on clustering and stratified sampling to increase scalability, which we then use for real world data-sets. Importantly, we provide near global optimality guarantees for our approach and show experimentally that our solution quality is better than the locally optimal ones achieved by state-of-the-art gradient-based methods. We experimentally compare our different approaches and baselines, and reveal nuanced properties of a DRO solution.
\end{abstract}

\section{Introduction}
Distributionally robust optimization (DRO) is a popular approach employed in robust machine learning. Mostly, if not always, these works have focussed on the task of classification or regression. However, often in practical applications the end goal of learning is a decision output $\bz$, which requires yet another complex optimization that uses the output $\widehat{\bx}_1, \ldots, \widehat{\bx}_N$ of a regressor $f(\cdot)$. For example, in facility location problem the learning of facility values is followed by an optimization using the values predicted to decide where to locate facilities and in security games adversary behavior model is learned and then an optimal defense allocation computed based on the learned model. Often the learning output is provided as public datasets with no access to the underlying private dataset used for such learning. In this set-up, we aim to provide robustness at the decision making level with access to only the non-robust learning output $\widehat{\bx}_1, \ldots, \widehat{\bx}_N$.

However, often the objective $F(\bz, \bx)$ of decision optimization is \emph{non-convex} in the learning system output $\bx$ unlike the convex objective of classification or regression, presenting significant scalability challenges. In general, for decision making and specifically for the problem domains we consider, \emph{global optimality is important} as sub-optimal decisions can lead to large revenue loss or loss of life; thus, the local optimality provided by gradient based methods is not sufficient.
As a consequence, in this paper, we study the scenario of calculating DRO decisions using the given multi-dimensional real valued outputs $\widehat{\bx}_1, \ldots, \widehat{\bx}_N$ of a non-robust learned $f$. A first result (Theorem~\ref{thm:thm1}) characterizes the quality of DRO decision output compared to the scenario where we know the true $f^*$. 
Our main focus is on addressing the scalability issue for the DRO decision making problem for a particular, but widely used, class of \emph{sum of fractionals non-convex objective}.
This objective arises from the well-known \emph{discrete choice models}~\cite{Trai03} of human behavior, which is known to \emph{not} have scalable globally optimal solutions~\citep{schaible1995fractional,li2019product}; we use this for tackling two different decision optimization problem in facility location and a robust version of Bayesian Stackelberg security games problem with quantal response.  
As far as we know, this is a first attempt to solve the aforementioned non-convex problems in a DRO setting to near global optimality.

Our \emph{first contribution} is a modelling construct, where we reformualate the variance regularized form~\cite{duchi2019variance} of our non-convex sum of fractionals objective as a mixed integer second order cone program (MISOCP). While the MISOCP form provides more scalability than the original formulation and guaranteed solution quality (Theorem~\ref{piecewiselinearerror}), it still does not scale to real world sized datasets.
Our \emph{second contribution} is a pair of approaches that achieves further scalability by splitting the problem space into sub-regions and solving a smaller MISOCP over representative samples from the sub-regions. 
Under mild conditions, both approaches provide \emph{global optimality guarantees} (Theorem~\ref{the:cluster},~\ref{the:strata}). 

Our \emph{final contribution} is detailed experiments validating the scalability of our approaches on a simulated security game problem as well as two variants of location facility using park and ride data-sets from New York~\citep{holguin2012new}. We compare with two gradient-based approaches~\citep{lin2020gradient} and show the superior solution quality achieved by our approach, which also reveals the need for global optimality. We further show a nuanced property of the DRO solution in providing better decisions for low probability scenarios over non-robust versions. Overall, our work 
provides desired \emph{robustness with globally optimal solution guarantees}. 

\textbf{Related work:} 
Our work is built on a recent line of research that connects the concepts of DRO and variance regularization \citep{duchi2019variance,duchi2021statistics,lam2016robust,Maurer2009EmpiricalBB,staib2019distributionally}. While most the previous studies along this research line focus on convex and continuous problems or problems with submodular objectives,
our work concerns a class of DRO problems with fractional structures, which are highly non-convex and requires new technical developments for globally optimal solution. Recent work~\cite{yan2020stochastic,qi2021online} has addressed non-convex objectives in DRO using gradient based methods that converge to stationary points, which is insufficient for decision making as we experimentally show that stationary points and globally optimal points can yield very different decision utilities. 

The literature on DRO is vast  and we refer the reader to \cite{rahimian2019distributionally} for a review. DRO methods can be classified by different ways to define ambiguity sets of distributions, for instance, ambiguity sets based on $\phi$-divergences \citep{ben2013robust,duchi2019variance,staib2019distributionally} or  Wasserstein distances \citep{pflug2007ambiguity,esfahani2018data,shafieezadeh2015distributionally,blanchet2019quantifying}. 
In this work, we focus on  $\phi$-divergence based models, motivated by their interesting connections with variance regularization and the tractability of the resulting non-convex DRO models.  

We show that our DRO methods can be used in some popular decision-making problems such as  Stackelberg security game (SSG) with Quantal Response ~\citep{tambe2011security,xu2016mysteries,fang2017paws,sinha2018stackelberg,yang2012computing,1501625} or competitive facility location  under random utilities ~\citep{Benati2002maximum,Freire2016branch,Mai2020multicut,DAM2021}. To the best of our knowledge, a DRO Bayesian model has not been studied in existing SSG works. In the context of competitive facility location under random utilities, we seem to be the first to bring DRO as a consideration. We handle a DRO version of a facility cost optimization problem, which has also never been studied in prior work.

\section{Background, Preliminary Notation and Result}
We use bold fonts for vectors and non-bold font for vector components and scalars, e.g., $x_j$ is a component of $\bx$. $[N]$ denotes $\{1, \ldots, N\}$. A $d$-dimensional vector is written as $\bx = (x_j)_{j \in [d]}$ or as $(x_1, \ldots, x_d)$. The positive part of a vector $\bx$ is $\bx^+ = (\max (0, x_j))_{j \in [d]} $, and the negative part is $\bx^- = (\min (0, x_j))_{j \in [d]} $. $\textbf{0}, \textbf{1}$ represent the all zero and all one vector.

\textbf{Distributionally Robust Optimization}:
Consider 
a function $F$ with inputs being a decision variable $\bz$ and parameter $\bx \in X$. Both $\bz$ and $\bx$ lie in an Euclidean space and both \emph{are constrained by linear constraints}; for notational ease we skip writing the constraints in the general formulation.
We seek to maximize the following objective function 
$
\max_{\bz} \bbE_{P}[ F(\bz,\bx )]
$.
For many classes of distributions the above is generally not tractable and one needs to sample $\bx$ from $P$. Let $\widehat{\bx}_1 ,\ldots, \widehat{\bx}_N$ be $N$ samples, we can solve the sample average approximation (SAA) problem instead
$
\max_{\bz}  \sum_{n\in [N]}F(\bz,\widehat{\bx}_n)$.
Let $\widehat{P}_N$ be the empirical distribution induced by the samples. The SAA above is same as $\max_{\bz}  \bbE_{\widehat{P}_N}[F(\bz,\bx)]$. A
 distributionally robust version of the SAA problem is
$
    \max_{\bz} \min_{\widetilde{P} \in \cP_{\xi,n}} \left\{ \bbE_{\widetilde{P}} [F(\bz,\bx)]\right\}
$, 
where the ambiguity set
$
\cP_{\xi,n} = \left\{\widetilde{P}|\; \cD_{\phi}(\widetilde{P} || \widehat{P}_N) \leq \xi/N \right\}
$,
and $\cD_{\phi}(P|| Q)$ is the $\chi^2$ divergence: $
\cD_{\phi}(P|| Q) = \frac{1}{2}\int (dP/dQ-1)^2 dQ
$. The above optimization can be written equivalently as ($\Delta_{\xi,n}$ defined below)
\begin{equation}
\label{prob:DRO}\tag{\sf\small DRO}
    \max_{\bz} \min_{\bp \in \Delta_{\xi,n}}\Big\{ \sum_{i\in [N]} p_i F(\bz,\widehat{\bx}_i)\Big\}
\end{equation}
where  
$
\Delta_{\xi,n} = \Big\{\bp \in  \bbR^N_+\Big|\; \sum_{i}p_i = 1;\; ||\bp - \textbf{1}/N||^2_2 \leq 2\frac{\xi}{N^2}\Big\}
$. We have earlier stated that $\widehat{\bx}_i$ is output by a regressor, say $f \in \cF$ for some function class $\cF$ trained using loss $\cL$ with $N_T$ datapoints, but this implies that $\widehat{\bx}_i = f(b_i)$ might not exactly same as $\bx^*_i = f^*(b_i)$ for some underlying feature values $b_i$ and best function $f^* \in \cF$. We assume $f^*$ is deterministic and has zero Bayes risk. Let $P^*$ be the true distribution induced by $f^*$ from which the (unknown) samples $\bx^*$'s are obtained; hence, the true utility of a decision $\bz$ is $\bbE_{P^*}[F(\bz, \bx)]$. We prove an end to end guarantee about the output decision $\widehat{\bz}^{**}$ using $\widehat{\bx}_i$'s, which reveals that $\widehat{\bz}^{**}$ is not much worse than the decision $\bz^{**}$ learned if $\bx^*$ would be available and used, and larger training data $N_T$ helps.

\begin{theorem} \label{thm:thm1}
If the optimal decision when solving \ref{prob:DRO} is $\bz^{**}$ using $\bx^*_i$'s and $\widehat{\bz}^{**}$ using $\widehat{\bx}_i$'s, $F$ is $\tau$-Lipschitz in $\bx$, $X$ is bounded, and a scaled $\cL$ upper bounds $||\cdot||_2$ (i.e., $||\bx-\bx'||_2 \leq \max(k\cL(\bx,\bx'), \epsilon)$ for constants $k, \epsilon$) then, 
the following holds with probability $1-2\delta - 2 \delta_1$: $\bbE_{P^*}[ F(\widehat{\bz}^{**},\bx )] \geq \bbE_{P^*}[ F(\bz^{**},\bx )] - C/\sqrt{N} - (1 + 2\sqrt{\xi})\tau \epsilon - \epsilon_N - \epsilon_{N_T}$,
where $\epsilon_{K} = C_1 \cR_{K}(\cL \circ \cH) + C_2/\sqrt{K}$ and $\cR_{K}$ is the Rademacher complexity with $K$ samples and $C, C_1, C_2$ are constants dependent on $\delta, \delta_1, \xi, k, \tau$.
\end{theorem}


\textbf{Variance Regularizer}:
As a large number of samples are needed for a low variance approximation of the true distribution, another proposed  robust version of the SAA~\citep{Maurer2009EmpiricalBB,duchi2019variance} is to optimize the following variance-regularized (VR) objective function
\begin{equation}
\label{prob:VR}\tag{\sf VR}
\max_{\bz}  \left\{\bbE_{\widehat{P}_N}[F(\bz,\bx)] - C \sqrt{\frac{\text{Var}_{\widehat{P}_N}(F(\bz,\bx))}{N}}\right\}.   
\end{equation}
The above allows to directly
optimize the trade-off between bias and variance. In a fundamental result, \citet[Theorem 1]{duchi2019variance} show that, with high probability, problem~\eqref{prob:VR} is \emph{equivalent} to the problem~\eqref{prob:DRO}.
Further, \citet{duchi2019variance} argue for solving the DRO version of the problem for concave $F$ (note we are solving a maximization SAA problem) since concave $F$ results in concavity of $\min_{\bp \in \Delta_{\xi,n}} \sum_{i\in [N]} p_i F(\bz,\widehat{\bx}_i)$, thus, the overall DRO problem is a concave maximization problem. In contrast, the objective in \eqref{prob:VR} is not concave.

In this paper, our focus is on $F$ that is \emph{not concave}, thus, the choice of DRO or variance regularized form is not obvious. For the class of functions $F$ that we analyze, we argue the variance regularized version is more promising as far as scalability for global optimality is concerned. We work with the assumption that the variance regularized form is equivalent to DRO, which holds under the mild condition shown in Equation~(9) in \citet{duchi2019variance}.

\section{Towards a Globally Optimal Solution}
In this section, we present results for a general class of non-concave functions $F$ that has a fractional form with non-linear numerator and denominator and that can be approximated by a linear fractional form with binary variables.
Then, we show three \emph{prominent} applications of our approach. \\
\textbf{Notation}: For ease of notation, we use shorthand to denote $F(\bz,\widehat{\bx}_i)$ by $F_i$ and $2 \frac{\xi}{N^2}$ by $\rho$.

\subsection{General Recipe to Form a MISOCP}
We perform a sequence of variable and constraint transformations of \eqref{prob:VR}, leading to a MISOCP.

\textbf{Mixed Integer Concave Program}:
The variance regularized objective in shorthand notation is: 
\begin{align}
   \cG(\bz) = \sum_i \frac{F_i}{N} - \sqrt{\rho\sum_i \Big (\frac{\sum_i F_i}{N} - F_i \Big )^2} \label{eq:shorthandobjective}
\end{align}

We substitute $l_i = \frac{\sum_i F_i}{N} - F_i$ and $q = \frac{\sum_i F_i}{N}$ for all $i  \in [N]$, such that 
$\sum_i l_i = 0$ and $F_i = q - l_i$. The objective in Equation~\ref{eq:shorthandobjective} thus becomes $q - \sqrt{\rho \sum_i l_i^2}$ which is concave in the variables $q$ and $\{l_i\}$. We add the new constraints $\sum_i l_i = 0$ and $F_i = q - l_i$ for all $i \in [N]$. Note that, while the objective is now concave with above changes, we have pushed the non-convexity into the constraints $F_i - q + l_i = 0$ for all $i  \in [N]$ that are added to the optimization.

If $F_i$ can be written (or approximated) as a fraction with affine numerator and denominator, we can convert the constraint $F_i - q + l_i = 0$ into a convex constraint, giving us an overall concave program. The conversion is explained next. Suppose $F_i$ can be written (or approximated) as $\frac{\ba_i ^T \bv  + b_i}{\ba'^T_i \bv + b'_i}$ where $\bv$ represents \emph{binary} variables after conversion ($\bv$ completely replaces $\bz$ and  $\ba_i,\ba_i',b_i, b_i'$ are dependent on $\widehat{\bx}_i$'s). 
Typically, such a linear fractional form is constructed via a piece wise linear approximation of the original non-linear numerator and denominator of $F$. Assume $\bv$ is of dimension $d$; typically $d$ will depend on the number of pieces. Define $\by_i = \bv t_i$ where $t_i = \frac{1}{\ba_i'^T \bv + b'_i}$. Then, we can (re)write the fractional form for $F_i$ as $F_i = \ba_i ^T \by_i + b_i t_i$. This yields the linear constraints below with the non-linearity now restricted to $\by_i = \bv t_i$.
\noindent\begin{minipage}{.45\linewidth}
\begin{align}
    &\sum_{i=1}^N l_i = 0 & \label{eq:sumli}
\end{align}
\end{minipage}%
~~~~~\vline%
\begin{minipage}{.5\linewidth}
\begin{align}
    & \ba_i^T\by_i + b_it_i = q - l_i  & \forall i \in [N]\\
    & \ba_i^{'T} \by_i + b_i't_i - 1 = 0  & \forall i \in [N]
\end{align}
\end{minipage}

We handle $\by_i = \bv t_i$ using McCormick relaxation technique~\citep{mccormick1976computability}. Typically, McCormick relaxation is applied for bilinear terms that are the product of two continuous variables, in which case, it is an approximation. However, in our case since $\bv$ is a binary vector variable, the McCormick relaxation yields an exact reformulation of the bilinear term. For applying McCormick technique, we need an upper and lower bound of $\bv$ and $t_i$. Since $\bv$ a vector of binary variables, we have lower bound $\bv^L = \textbf{0}$ and upper bound $\bv^U = \textbf{1}$. 
Similarly, $t_i^L = \frac{1}{(\ba_i'^{+})^T \textbf{1} + b_i'}$ and $t_i^U = \frac{1}{(\ba_i'^{-})^T\textbf{1} + b_i'}$ (recall superscript $+$ and $-$ indicate positive and negative part of a vector respectively). Further, it is assumed $t_i^U$ and $t_i^L$ exist. 
Note that these bounds are not variables but fixed constants that depend on the fixed parameters $\ba_i,\ba_i',b_i,b_i'$, hence these need to be computed just once.
Using the upper and lower bounds of $\bv$ and $t_i$ in McCormick technique we get: 

\noindent\begin{minipage}{.5\linewidth}
\begin{align}
    &\by_i - \bv t_i^U \leq 0; & \forall i \in [N] \label{eq:firstMcC}\\
         &\by_i - (\textbf{1}t_i + \bv t_i^L - \textbf{1}t_i^L) \leq 0; & \forall i \in [N]\\
         &-\by_i + (\textbf{1}t_i + \bv t_i^U - \textbf{1}t_i^U) \leq 0; & \forall i \in [N]
\end{align}
\end{minipage}%
~~~~~\vline%
\begin{minipage}{.45\linewidth}
\begin{align}
         &-\by_i + \bv t_i^L \leq 0; & \forall i \in [N]\\
         & \bv \in \{0, 1\}^{d}\\
         &t_i^U \leq t_i \leq t_i^L; & \forall i \in [N] \label{eq:lastMcC}
\end{align}
\end{minipage}

It is straightforward to check the above set of equations is equivalent to $\by_i = \bv t_i$. With the changes, we obtain a mixed integer concave program (with all constraints linear). Next, while the above can be solved using branch and bound with general purpose convex solvers for intermediate problem, we show that a further transformation to a MISOCP is possible. Specialized SOCP's solvers provide much more scalability than a general purpose convex solvers~\citep{bonami2015recent} and hence partially address the scalability challenge.

\textbf{Mixed Integer SOCP}:
We transform further by introducing another variable $s$ to stand for $\sqrt{\rho \sum_i l_i^2}$. We use $\br = (s,q,(l_i)_{i \in [N]}, \bv, (t_i)_{i \in [N]}, \by_1, \ldots,  \by_N)$ to denote all the variables of the optimization. 
Thus, the objective becomes the linear function $q - s$ with an additional constraint that \begin{align}
    \sqrt{\rho \sum_i l_i^2} \leq s \label{eq:SOCP}
\end{align}
The above is same as $\vert \vert  A \br \vert \vert_2 \leq \bc^T \br  $ for the constant matrix $A$ (with entries $0$ or $\sqrt{\rho}$) and constant vector $\bc$ (with 1 in the $s$ component, rest 0's) that picks the $l_i$'s and $s$ respectively. 
This is a SOCP form of constraint, and the linear objective $q-s$ makes the problem after this transformation a MISOCP.
In the above reformulation, the only approximation is introduced in writing $F_i$ as a linear fractional term. 
One way of such approximation is via piecewise linear approximation (PWLA). Suppose the fraction function $F(\bz, \widehat{\bx}_i)$ has a separable (in $\bz$) numerator and denominator of the form $\frac{\sum_j n(z_j, \widehat{\bx}_i)}{\sum_j d(z_j, \widehat{\bx}_i)}$ where $j$ ranges over the components of $\bz$, and $n(z_j, \widehat{\bx}_i)$ and $d(z_j, \widehat{\bx}_i)$ are non-negative and Lipschitz continuous.
In this case, a general PWLA is possible with the following guarantee:
\begin{theorem} \label{piecewiselinearerror}
For $F(\bz, \widehat{\bx}_i) = \frac{\sum_j n(z_j, \widehat{\bx}_i)}{\sum_j d(z_j, \widehat{\bx}_i)}$ as stated above and approximated as $\frac{\ba_i ^T \bv  + b_i}{\ba'^T_i \bv + b'_i}$, a PWLA approximation with $K$ pieces yields 
$\vert \cG(\bz^*) - \cG(\widehat{\bz}^{**}) \vert \leq O(1/K)$, where $\cG(\bz^*)$ and $\cG(\widehat{\bz}^{**})$ are the optimal objective values with approximation (MISOCP) and without the approximation respectively.
\end{theorem}

Next, we show instantiation of the just presented general recipe for three widely studied problems. 

\subsection{Applications}
\textbf{Notation}: In the SSG (facility location) application $m$ resources (facility) are allocated to $M$
targets (locations). $\bx$ maps to type $\theta^a_{\bx}$ of adversary, and type $\theta^d_{\bx}$ of defender in SSG, and type $\theta_{\bx}$ of clients of facility or directly $V_{\bx}$ utility for each client type in facility location. 

\textbf{Bayesian Stackelberg Security Game with Quantal Response}:
A SSG models a Stackelberg game where a defender moves first to allocate $m$ security resources for protecting $M$ targets. The randomized allocation is specified by decision variables $\bz$ of dimension $M$ with the constraints that $\sum_{i=1}^M z_i \leq m$ ($z_i \in [0,1]$); $z_i$ is interpreted as the protection probability of the target $i$. Past works have used the model of a quantal responding adversary~\citep{sinha2018stackelberg}].
We generalize this to a Bayesian game version where there is a continuum of attackers types with the type specified by a parameter $\bx$ and an unknown prior distribution over these types. The attacker's utility in attacking the target $j$ is a function of the protection probability of target $j$ and type: $h(z_j, \theta^a_{\bx})$. Similarly, the defender's utility when target $j$ is attacked is: $u(z_j, \theta^d_{\bx})$ for some player-specific parameters $\theta$ that depend on $\bx$. Following quantal response model (for attacker only), the attacker of type $\bx$ attacks a target $j$ with probability $\frac{\exp(h(z_j, \theta^a_\bx))}{\sum_{j \in [M]} \exp(h(x_j, \theta^a_{\bx}))}$ and the defender utility is $F(\bz, {\bx}) =  \frac{\sum_{j \in [M]} u(z_j, \theta^d_\bx)\exp(h(z_j, \theta^a_\bx))}{\sum_{j \in [M]} \exp(h(x_j, \theta^a_\bx))}$.
Note that in case the defender's utilities $u(z_j, \theta^d_\bx)$ take negative values and the assumptions of Theorem~\ref{piecewiselinearerror} will be violated. This issue can be simply fixed by choosing $\alpha$ such that $\alpha \geq \max_{\bz,\bx} \{-u(z_j, \theta^d_\bx)\}$  and replacing $F(\bz,\bx)$ by  $F(\bz, {\bx}) +\alpha =  \frac{\sum_{j \in [M]} (u(z_j, \theta^d_\bx)+\alpha) \exp(h(z_j, \theta^a_\bx))}{\sum_{j \in [M]} \exp(h(x_j, \theta^a_\bx))}$. 
This will make all the numerators and enumerators of the objective function non-negative, while keeping the same optimization problem.  We also note that quantal response is also known as multinomial logit model in the discrete choice model literature~\citep{Trai03}. Our generalization here to multiple types of adversary makes the problem akin to the mixed logit model in discrete choice models, which is generally considered intractable. As a consequence, our solution addresses a basic problem in discrete choice models also.

Following our set-up, we observe $N$ samples of the types of attackers $\widehat{\bx}_1, \ldots, \widehat{\bx}_N$ (which gives $\widehat{\theta}^a_1, \ldots, \widehat{\theta}^a_N, \widehat{\theta}^d_1, \ldots, \widehat{\theta}^d_N$) and we solve a robust version of the problem. Further,
following our general recipe for solving the robust problem, we piecewise approximate the numerator and denominator of $F$ using $K$ pieces, where the dimension of $\bv$ is $d = MK$. For this approximation, we require two additional linear constraints over the constraints in Equations~(\ref{eq:sumli}\mbox{-}\ref{eq:SOCP}). The optimization then is: 
\begin{align*}
    &\max_{\br}  \; q - s \label{prob:SSG}\tag{\sf\small SSG}\\
 & \mbox{subject to  Constraints}~(\ref{eq:sumli}\mbox{-}\ref{eq:SOCP}), \sum_{j \in [M]} \sum_{k \in [K]} v_{jk} - mK \leq 0,   v_{j,k} \geq v_{j,k+1}; \quad \forall k \in [K]
\end{align*}
The overall additive solution bound of $O(1/K)$ can be readily inferred from Theorem~\ref{piecewiselinearerror}.




\textbf{Max-Capture Competitive Facility Location (\textbf{MC-FLP})}:
In this problem~\citep{Mai2020multicut}, a firm has $M$ locations ($[M]$) to set up at most $m < M$ facilities. The aim is to maximize the number of clients using this firm's facilities. The competitor(s) already have facilities running at locations $Y \subset [M]$. 
There are different types of clients, where types are denoted by $\bx$. The number of clients of type $\bx$ is known and equal to $s_{\bx}$. However, the distribution over types is \emph{unknown}. The firm's decision of which location to choose is given by binary variables $z_j \in \{0,1\}$ for $j \in [M]$. 
A utility of any client of type $\bx$ for visiting location $j$ is $V_{\bx, j}$. The \emph{choice probability} of a client of type $\bx$ choosing any of this firm's location is given as a quantal response model
$
\frac{\sum_{j \in [M]} z_j e^{V_{\bx, j}}}{\sum_{j \in [M]} z_j e^{V_{\bx, j}} + \sum_{j \in Y} e^{V_{\bx, j}}}
$.
For shorthand, we abuse notation and use $V_{\bx, j}$ to replace $e^{V_{\bx, j}}$ and $U_{\bx, Y}$ to replace $\sum_{j \in Y} e^{V_{\bx,  j}}$. This gives $F(\bz, \bx) =  \frac{s_\bx \sum_{j \in [M]} z_j V_{\bx, j}}{\sum_{j \in [M]} z_j V_{\bx, j} + U_{\bx, Y}}$, which is interpreted as the expected number of clients of type $\bx$ choosing this firm's facilities.

Following our set-up, we observe $N$ samples of the types of clients samples $(\widehat{V}_{1,j})_{j \in [M]}, \ldots, (\widehat{V}_{N,j})_{j \in [M]}$ and we solve a robust version of the problem. Here, we get $F_i = \frac{s_i \sum_{j \in [M]} z_j \widehat{V}_{i, j}}{\sum_{j \in [M]} z_j \widehat{V}_{i, j} + \widehat{U}_{i, Y}}$.
Next,
following our general recipe for solving the robust problem, we note that $F_i$ is already in the form $\frac{\ba_i ^T \bv  + b_i}{\ba'^T_i \bv + b'_i}$ where $\bz$ plays the role of $\bv$. Thus, the dimension of $\bv$ is $M$ and no approximation is needed here for $F_i$. Thus, by Theorem~\ref{piecewiselinearerror}, we achieve the global optimal solution by solving the MISOCP optimally. The full MISOCP with an additional number of location constraint is (note $\bv$ directly replaces $\bz$): 
\begin{align*}
        \max_{\br} & \; q - s \; \mbox{ subject to }\mbox{Constraints}~(\ref{eq:sumli}\mbox{-}\ref{eq:SOCP}), \sum_{j \in [M]} v_j - m \leq 0.
\end{align*}


\textbf{Max-Capture Facility Cost Optimization (\textbf{MC-FCP})}:
In the previous \textbf{MC-FLP} problem, the budget was specified as a constraint on the number of facilities. However, often a more realistic set-up is where there is a monetary constraint and the attractiveness of a facility depends on the investment into the facility.
Thus, modifying the previous problem slightly, $z_j$ takes a different meaning of the amount of investment into facility at location $j$ (zero investment indicates no facility). Given this investment, the attractiveness of a facility $j$ for the client of type $\bx$ is given as $h(z_j,  \theta_{\bx,j})$ for some parameter $\theta$ dependent on $\bx$ and $j$. And the \emph{ choice probability} of a client of type $\bx$ choosing any of this firm's location is given as a quantal response model
$\frac{\sum_{j \in [M]} e^{h(z_j, \theta_{\bx,j})} }{\sum_{j \in [M]} e^{h(z_j, \theta_{\bx,j})} + U_{\bx, Y}} $. This gives $F(\bz, \bx) =  \frac{s_\bx \sum_{j \in [M]} e^{h(z_j, \theta_{\bx,j})}}{\sum_{j \in [M]} e^{h(z_j, \theta_{\bx,j})} + U_{\bx,Y}}$, which is interpreted similar to \textbf{MC-FLP}. As stated, we observe $N$ samples of the types of clients which gives $ (\widehat{\theta}_{1,j})_{j \in [M]}, \ldots, (\widehat{\theta}_{N,j})_{j \in [M]}$ and we solve a robust version of the problem. Here, we get $F_i = \frac{s_i \sum_{j \in [M]} e^{h(z_j, \widehat{\theta}_{i,j})}}{\sum_{j \in [M]} e^{h(z_j, \widehat{\theta}_{i,j})} + \widehat{U}_{i,Y}}$.
Next,
as can be seen from the form, this is similar to the \textbf{SSG} problem and, following our general recipe, with two additional linear constraints the optimization formulation is exactly same as Equation~\eqref{prob:SSG}.

\section{Scaling up in Number of Samples}
The transformation to a MISOCP helps in scalability over a general mixed integer concave program, but for real world dataset sizes (e.g., $80,000$ data points in our experiments) we need further scalability. We explore two related techniques towards this end: clustering and stratified sampling. For both approaches, we obtain a representative subset of $S$ data points ($S <\!< N$) and a modified weighted objective, which converts to a much smaller tractable MISOCP compared to the original problem. For solution guarantees, we need mild assumptions: in particular, for the rest of this section we assume a bounded $F$, i.e., for some fixed $\psi$ $
    \max_{\bz} \{F(\bz, \bx)\} - \min_{\bz} \{F(\bz, \bx)\} \leq \psi^2 \quad \forall \widehat{\bx}_1, \ldots, \widehat{\bx}_N 
$ 
and $\tau$-lipschitzness of $F$ in the argument $\bx$:
$
    \vert F(\bz, \bx') - F(\bz, \bx) \vert \leq \tau \vert\vert\bx' - \bx\vert\vert_2 \quad \forall \bz 
$.

\textbf{Clustering Approach}:
We cluster the $N$ points $\bx_1,...\bx_N$ into $S$ groups and for each group $s$ we have $||\bx_i - \bx^s|| \leq \epsilon$, where $\bx^s$ is the cluster center of cluster $s$. We call $\epsilon$ the clustering radius. Let $C_s$ be the number of points in the cluster $s$, hence $\sum_{s \in [S]} C_s = N$. 
We use a shorthand for the original objective function of the MISOCP $\cG(\bz)$: 
\begin{align*}
& \sum_i \frac{F(\bz,\widehat{\bx}_i)}{N}  - \sqrt{\rho\sum_i \left(\sum_i \frac{F(\bz,\widehat{\bx}_i)}{N} - F(\bz,\widehat{\bx}_i)\right)^2 } = \widehat{\text{Mean}}(F(\bz,\bx))  - \sqrt{\rho \widehat{\text{Var}}(F(\bz,\bx))}
\end{align*}
where $\widehat{\text{Mean}}$ is empirical mean and $\widehat{\text{Var}}$ is \emph{unnormalized variance}.
After clustering, we solve for the same problem but only with cluster centers and appropriate \emph{weighing}, to get modified objective $\widehat{\cG}(\bz)$:
\begin{align}
\sum_s C_s \frac{F(\bz,\bx^s)}{N} - \sqrt{\rho\sum_s C_s \left( \sum_s C_s \frac{F(\bz,\bx^s)}{N} - F(\bz,\bx^s)\right)^2 }  = \widehat{\text{Mean}}^S(F(\bz,\bx)) - \sqrt{\rho \widehat{\text{Var}}^S(F(\bz, \bx))} \label{obj:clustering}
\end{align}
The conversion to MISOCP is exactly the same, except for $F_i$'s being weighted as shown above; details of conversion are in the appendix. 
We bound the approximation incurred by the two terms above (weighted mean and unnormalized weighted variance)  separately below 

\begin{lemma}\label{lem:mean}
Under assumptions stated above, we have 
    $\left| \widehat{\text{Mean}}(F(\bz,\bx))  -  \widehat{\text{Mean}}^S(F(\bz,\bx)) \right| \leq \tau \epsilon$ and \nonumber 
    $ \left \vert \sqrt{\rho \widehat{\text{Var}}(F(\bz,\bx))} - \sqrt{\rho \widehat{\text{Var}}^S(F(\bz, \bx)) } \right \vert
    \leq 
    (\psi + \sqrt{2\tau\epsilon} ) \sqrt{\frac{2\tau \epsilon \xi}{N}}$.
\end{lemma}



The next result is obtained by using the lemma above

\begin{theorem} \label{the:cluster}
Given the assumptions stated above,
and $\widehat{\bz}$ an optimal solution for  $\max_{\bz}\widehat{\cG}(\bz)$ and $\bz^*$ optimal for MISOCP $\max_{\bz} \cG(\bz)$, the following holds:
$
 |\cG(\widehat{\bz}) - {\cG}(\bz^*)|\leq 2 (\tau \epsilon + \psi \sqrt{\frac{2\tau \epsilon \xi}{N}} + \frac{2\tau \epsilon \xi}{\sqrt{N}})
$.
\end{theorem}

\textbf{Stratified Sampling}:
Similar in spirit to clustering, the space of $\bx$ space is divided into $T$ strata. Each strata has $C_t$ samples, such that $\sum_{t \in [T]} C_t = N$. Next, distinct from the clustering approach,
we draw $N_t$ samples randomly from the $t^{th}$ stratum with a total of $\sum_t N_t = S$ samples (note same number of total samples $S$ as in clustering). For each stratum $t$ we have $||\bx_i - \bx_j|| \leq d_t$ for any $\bx_i, \bx_j$ in stratum $t$. We denote a random sample in stratum $t$ as $\widehat{\bx}^j$ where $j \in [N_t]$ (note superscript is to distinguish from the subscript used to index all the $\widehat{\bx}$'s). This approach is the preferred one if the clustering approach results in cluster centers that are not allowed as parameter values (e.g., cluster center may be fractional where $\bx$'s can only be integral).


Let $l_t = \frac{C_t}{N_t}$. Use $\widehat{Mean}^T(F(\bz, \bx))$ to stand for $\frac{1}{N} \sum_{t \in [T]}l_t  \sum_{j \in [N_t]}F(\bz,\widehat{\bx}^j)$ and $\widehat{Var}^T(\bz, \widehat{\bx})$ for $\sum_{t \in [T]}l_t  \sum_{j \in [N_t]} \left(\widehat{Mean}(F(\bz, \bx)) - F(\bz,\widehat{\bx}^j)\right)$.
After stratified sampling our modified weighted objective $\widehat{G}(\bz)$ is 
$   \widehat{Mean}^T(F(\bz, \bx)) - \sqrt{\rho \widehat{Var}^T(\bz, \widehat{\bx})}
$. Next, similar to clustering, bounds for $ \widehat{Mean}^T$ and $ \widehat{Var}^T$ but with high probability (Lemma~\ref{lem:stratamean},\ref{lem:stratavar} in appendix) lead to the main result:
\begin{theorem} \label{the:strata}
Let $D = \max_{\bz, \bx } |F(\bz, \bx)|$  for bounded function $F$. Given the assumptions stated above,
and $\widehat{\bz}$ an optimal solution for  $\max_{\bz}\widehat{\cG}(\bz)$ and $\bz^*$ optimal for MISOCP $\max_{\bz} \cG(\bz)$, and $N_* = \min_t {N_t}$, the following statement holds with probability $\geq 1 - 2 \sum_t \exp^{\frac{-2 \sqrt{N_*}\epsilon^2}{\tau^2 d_t^2}} - 4 \sum_t \exp^{\frac{-2 \sqrt{N_*} \epsilon^2}{4 \tau^2 d_t^2 D^2 }}$:
\begin{align*}
|\cG(\widehat{\bz}) - {\cG}(\bz^*)|\leq \frac{2  \epsilon}{(N_*)^{1/4}} \Bigg (1 + 2\sqrt{\frac{ \xi}{\widehat{Var}(F(\bz,\bx))} } \Bigg ).
\end{align*}
\end{theorem}
Thus, with increasing samples in all strata, optimality gap approaches $0$ with prob. approaching~$1$.

\section{Experiments}

\begin{figure} 
    \centering
    \includegraphics[scale=0.47]{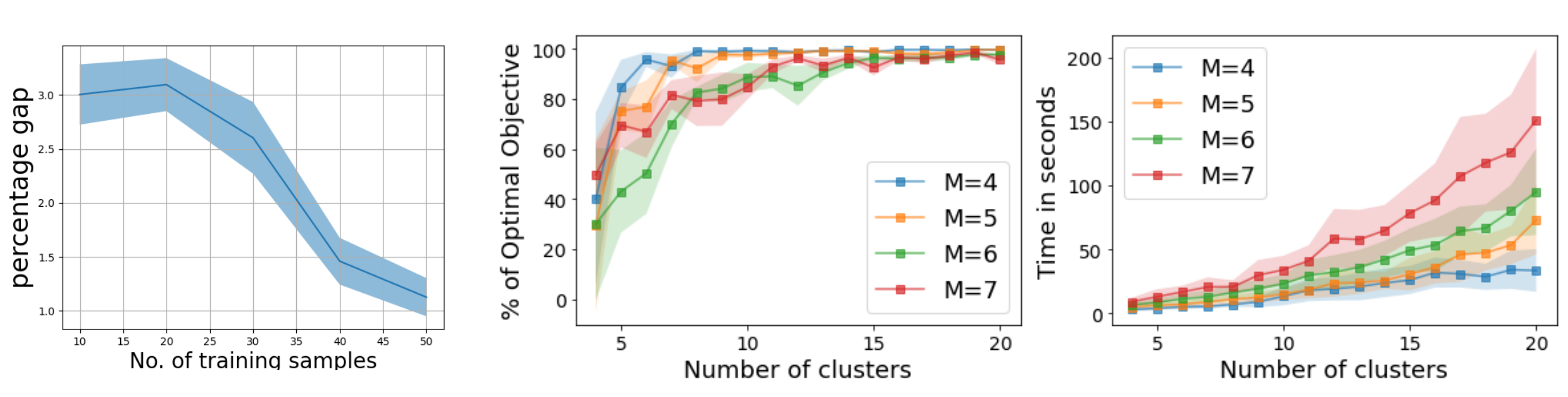}
    \caption{(Left) \% gap between the utility of decisions using true and learned regressor with varying training data size $N_T$. (Middle) Objective value achieved using clustering approach as a \% of \textbf{OPT}. (Right) Time to solve each optimization to optimality.
    Middle and right results are shown for varying alternatives number $M$. Underlying parameters are $N=500, m=1, \xi=$1E6.}
    \label{fig:SSG}
\end{figure}



\begin{table}[t]
    \caption{We use clustering/stratified sampling to approximately solve a problem (\textbf{Approx-OPT}). The table shows the mean percentage: $\frac{100 \times \textbf{Approx-OPT}}{\textbf{OPT}}$ and  standard deviation over 10 different synthetic SSG datasets with underlying parameters $N=500, M=6, m=1, K=10$.}
    \label{tab:cluster_vs_sampling}
    \centering
    \begin{small}
    \begin{tabular}{cccc}
    \toprule
     \multirow{2}*{Method} & \multicolumn{3}{c}{Total no. of samples (S)} \\
    \cline{2-4}
    & 8 & 16 & 24\\
    \midrule
        Clustering & 80.17 $\pm$ 2.14 & 90.74 $\pm$ 1.20 & 94.03 $\pm$ 0.57\\
         1 per strata & 85.49 $\pm$ 1.94 & 94.59 $\pm$ 0.76 & 99.25 $\pm$ 0.39 \\
         2 per strata & 91.78 $\pm$ 1.33  & 94.55 $\pm$ 0.64 & 99.54 $\pm$ 0.25\\
         4 per strata & 78.89 $\pm$ 1.73 & 94.86 $\pm$ 0.94 & 98.77 $\pm$ 0.39\\
         8 per strata & 92.19 $\pm$ 0.67 & 95.85 $\pm$ 0.80 & 99.37 $\pm$ 0.25\\
        \bottomrule
    \end{tabular}
    \end{small}
\end{table}

We evaluate our methods on (a) Stackleberg Security Games (\textbf{SSG}) with Quantal Response (synthetic data), (b) Maximum capture Facility Location Planning (\textbf{MC-FLP}) and (c) Maximum capture Facility Cost Planning (\textbf{MC-FCP}). Empirically we demonstrate (i) better solution quality of our method compared to baselines, (ii) practical scalability of our method, and (iii) improvement over non-robust optimization on those data points that contribute least to the objective (akin to rare classes in classification) while not sacrificing average performance. We fix $K=10$ in PWLA as we find that objective increase saturates for this $K$ (see Appendix~\ref{sec:K}). We use a 2.1 GHz CPU with 128GB RAM.


\begin{table}[t!]
    \caption{Objective values as a \% of \textbf{OPT} across various methods repeated over 5 synthetic SSG datasets with parameters $N=500, M=10, m=1$ for varying regularization ($\xi$, on left) and $N=500, m=1, \xi =$ 1E6 for varying no. of targets (M, on right). The no. of clusters/strata is 50.}
    \label{tab:baseline_xi}
    \centering
    \begin{small}
    \begin{tabular}{ccccc|ccc}
    \toprule
    \multirow{2}*{Method} & \multicolumn{4}{c}{Regularization ($\xi$)} & \multicolumn{3}{c}{No. of Alternatives (M)}\\
    \cmidrule{2-8}
    & 1E3 & 1E4 & 1E5 & 1E6 & 10 & 25 & 50\\
    \midrule
         TT-GAD& 99.8$\pm$0.1 & 99.4$\pm$0.1 & 92.7$\pm$0.3 & 82.6$\pm$0.4 & 82.6 $\pm$ 0.4 & 90.2 $\pm$ 0.5 & 92.2 $\pm$ 0.3\\
         PGA& 98.9$\pm$0.1 & 98.1$\pm$0.2 & 87.7$\pm$0.5 & 49.2$\pm$0.9 & 49.2 $\pm$ 0.9 & 90.9 $\pm$ 0.7 & 93.5 $\pm$ 0.4\\
         Clustering\!& 99.9$\pm$0.1 & 99.9$\pm$ 0.1 & 99.8$\pm$0.1 & 99.6$\pm$0.1& 99.6 $\pm$ 0.1 & 99.5 $\pm$ 0.2 & 99.4 $\pm$ 0.1\\
         Sampling\!& 100.0$\pm$0.0 & 99.9$\pm$ 0.1 & 99.9$\pm$0.1 & 99.8$\pm$0.1& 99.8 $\pm$ 0.1 & 99.6 $\pm$ 0.1 & 99.5 $\pm$ 0.2\\
    \bottomrule
    \end{tabular}
    \end{small}
    \vskip -0.1in
\end{table}


\begin{table*}[t]
    \caption{Average client choice probabilities for availing the facility across various settings. H denotes average over those 5\% of the clients in test data with the lowest choice probabilities, A denotes average over all the samples in the test set.}
    \label{tab:scores}
    \centering
    \begin{scriptsize}
    \begin{tabular}{ccccccc|cccccc}
    \toprule
    \multirow{3}*{$\xi$}& \multicolumn{6}{c}{\textbf{MC-FCP}} & \multicolumn{6}{c}{\textbf{MC-FLP}}\\ 
    \cline{2-13}
    & \multicolumn{2}{c}{m=7} & \multicolumn{2}{c}{m=10} & \multicolumn{2}{c|}{m=13} & \multicolumn{2}{c}{m=10} & \multicolumn{2}{c}{m=12} & \multicolumn{2}{c}{m=14}\\
    \cline{2-13}
    & H & A & H & A & H & A & H & A & H & A & H & A\\
         ERM & 0.069 & 0.692 & 0.150 & 0.719 & 0.420 & \textbf{0.758} & 0.175 & 0.741 & 0.426 & \textbf{0.769} & 0.469 & 0.772\\      
  1E2 & 0.069 & 0.692 & \textbf{0.417} & \textbf{0.751} & 0.4170 & 0.751 & \textbf{0.418} & \textbf{0.763} & 0.426 & \textbf{0.769} & 0.469 & 0.772 \\                                 
  1E3 & \textbf{0.093} & \textbf{0.697} & 0.416 & 0.747 & 0.417 & 0.757 & \textbf{0.418} & \textbf{0.763} & 0.425 & 0.768 & 0.533 & \textbf{0.777} \\                                  
 1E4 & \textbf{0.093} & \textbf{0.697} & 0.416 & 0.747 & \textbf{0.446} & 0.750 & 0.417 & 0.759 & \textbf{0.531} & 0.767 & \textbf{0.539} & 0.769\\ 
\bottomrule
    \end{tabular}
    \end{scriptsize}
    \vskip -0.1in
\end{table*}
\textbf{Baselines:} We use the following two methods as baselines: (i) Projected Gradient Ascent (\textbf{PGA}) on the formulation \eqref{prob:VR}, (ii) Two Time Scale Gradient Ascent Descent (\textbf{TT-GAD}) \citep{lin2020gradient} on the formulation \eqref{prob:DRO} where the inner minimization is convex and the outer maximization is non-concave. The numbers reported for our baselines are the \emph{best values} over 10 random initializations.

\subsection{SSG with Quantal Response (Synthetic Data)}
We generate attacker and defender utilities following~\cite{yang2012computing}; a complete description of data generation and choice of $f^*$ is in Appendix~\ref{sec:datagen}. We generate five datasets of size $N=500$ each in order to observe the variance of every result reported in this sub-section; this is also the largest size that we could solve exactly optimally within an hour using all the data points. First, we empirically validate Theorem \ref{thm:thm1} by plotting in Figure~\ref{fig:SSG}(left) the relative gap between the true utility $E_{P^*}[F(\cdot, \bx)]$ of the decisions output by running DRO on the output $(\bx^*)_{i \in [N]}$ of the true $f^*$ (assumed fixed linear function) versus on $(\widehat{\bx})_{i \in [N]}$ from learned $f$, as the training data size $N_T$ for learning $f$ is varied.

Next, we focus on only using $(\widehat{\bx})_{i \in [N]}$ and the optimal solution for $(\widehat{\bx})_{i \in [N]}$ is named as \textbf{OPT}.
Figure~\ref{fig:SSG} (middle) demonstrates empirically that the solution of the optimization problem on cluster centers converges to \textbf{OPT} with only a few number of clusters and the time for the optimization shown in Figure~\ref{fig:SSG} (right) is reasonable.    
Next, 
the results in Table \ref{tab:cluster_vs_sampling} show a comparison of the clustering and stratified sampling approach using the metric of how close they get to \textbf{OPT}. 
We find stratified sampling to be better than clustering in almost all cases.
Table~\ref{tab:baseline_xi} (left) demonstrates that the baselines struggle to reach the optimal value objective as the magnitude of regularization ($\xi$) increases. Intuitively, as $\xi$ increases the variance term (which is highly non-convex) contributes more to the objective and stationary points reached by the baselines are quite sub-optimal compared to the global optimal. In addition, with increasing $\xi$ the ambiguity set becomes larger possibly containing more local optimal solutions. We also study varying the parameter $M$ ($m$ fixed) and Table~\ref{tab:baseline_xi} (right) shows that our approaches outperform the gradient-based baselines across different values of $M$.

\begin{figure}[t]
    \centering
    \includegraphics[scale=0.47]{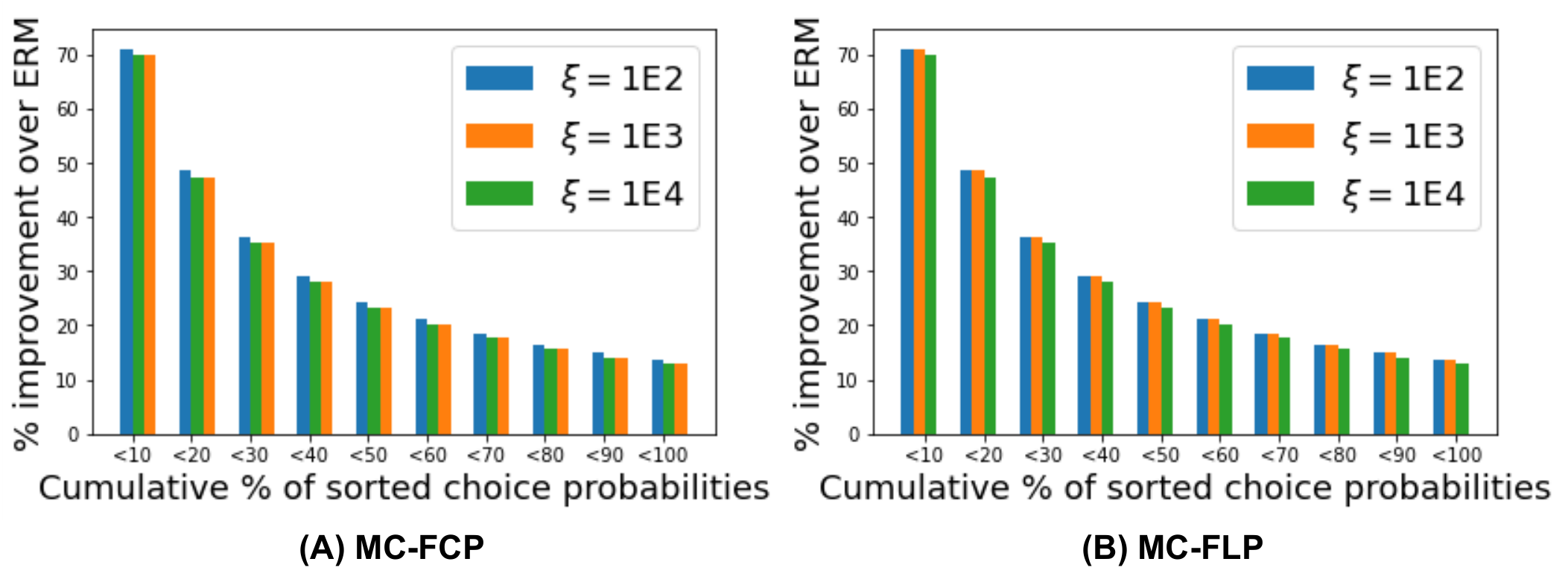}
    \caption{The bar plots show the percentage improvement of choice probabilities of our clustering approach over ERM over cumulative buckets of choice probabilities (see text) in ascending order with varying regularization 
    $\xi$ with fixed $m$=$10$.}
    \label{fig:cumulative}
    \vskip -0.05in
\end{figure}

\subsection{\textbf{MC-FLP} and \textbf{MC-FCP} (Real Data)}
\textbf{P\&R-NYC Dataset : } We use a large and challenging Park-and-ride (P\&R) dataset collected in New York City, which provides utilities for 82341 clients ($N$) for 59 park and ride locations ($M$), along with their incumbent utilities for competing facilities~\citep{holguin2012new}; this data was directly used for \textbf{MC-FLP}. For \textbf{MC-FCP} we additionally use generated costs, which are not present in the P\&R data. A complete description of data generation is in Appendix~\ref{sec:datagen}.
Both these problems could not be solved at all with our MISOCP alone (no clustering) as the optimization did not finish in 24 hours. Hence, we use our clustering approach with 50 clusters.

We compare to a baseline solution of the non-robust empirical risk minimization (ERM) (also called the sample average approximation or SAA). We split the data (randomly) into training and test (80:20) and then obtain the decision $\widehat{\bz}$ using the training data. Then, we obtain the choice probability (recall this as probability of a client choosing any of the firm's facility) for every client in the test data for the decision $\widehat{\bz}$. We compare the performance of ERM and our method for clients (in test set) bucketed by choice probabilities in
Figure~\ref{fig:cumulative} and Table~\ref{tab:scores}. The buckets are made by sorting all the clients in the test set by ascending choice probabilities and then considering cumulative buckets as the first 5\%, the first 10\% and so on.
In Fig.~\ref{fig:cumulative}, we show that the average  percentage improvement in choice probabilities of our robust approach over ERM is considerably higher for clients with lower choice probability (these clients contribute least to the objective) and the over all average  over all clients (rightmost on x-axis) is slightly better than ERM. In Table~\ref{tab:scores}, note the significantly increased probabilities for low choice probability clients (low choice prob. using $\widehat{\bz}$) without compromising the average performance across all clients for varying $M$. Additional results are in Appendix~\ref{sec:real_data}.

\section{Conclusion}
We presented an approach for a distributionally robust solution to a class of non-convex sum of fractional solutions, with guaranteed near global optimality. We presented application to three prominent practical problems and the connection to discrete choice models likely opens up possibilities of applying our approach to even more problems. Our scalability is limited to hundred thousand data points, further investigation on how to cluster or stratify more effectively (than k-means) to achieve even more scalability is a possible future research direction. We hope that our work inspires tackling robust formulation of more classes of non-convex problems, with guarantees for global optimality.

\bibliography{refs}
\bibliographystyle{plainnat}

\newpage
\appendix

\section{Piece-wise Linear Approximation (PWLA)}
\label{appd:PWLA}
Recall the general functional form $$F(\bz,\widehat{\bx}_i) = \frac{\sum_j n(z_j, \widehat{\bx}_i )}{\sum_j d(z_j, \widehat{\bx}_i)}$$ and note that both the numerator and denominator are separable in the components of the decision variables $\bz$. 
Let assume each variable $z_j$ can vary in the interval $[L_j,U_j]$, the idea here is to divide each interval $[L_j,U_j]$ into $K$ equal sub-intervals of size $(U_j-L_j)/K$ and approximate $z_j$ by $K$ binary variables $v_{jk} \in \{0,1\}$ as
\[
z_j =L_j \frac{U_j-L_j}{K}\sum_{k\in[K]} v_{jk},
\]
where $v_{jk} \in \{0,1\}$ satisfying $
v_{ik} \geq v_{j,{k+1}}$ for $k = 1,2,\ldots,K-1$. We then approximate $n(z_j, \widehat{\bx}_i)$ and $d(z_j, \widehat{\bx}_i)$ as 
\begin{align}
   n(z_j, \widehat{\bx}_i) \approx  \widehat{n}(z_j, \widehat{\bx}_i)  = n\left(L_j+ \floor*{z_jK/(U_j-L_j)}\frac{U_j-L_j}{K}, \widehat{\bx}_i\right) =  n(L_j, \widehat{\bx}_i) + \frac{U_j-L_j}{K}\sum_{k\in [K]}  \gamma^{ni}_{jk}{v_{jk}}, \nonumber\\
d(z_j, \widehat{\bx}_i) \approx  \widehat{d}(z_j, \widehat{\bx}_i) = d\left(L_j+ \floor*{z_jK/(U_j-L_j)}\frac{U_j-L_j}{K}, \widehat{\bx}_i\right) =  d(L_j, \widehat{\bx}_i) + \frac{U_j-L_j}{K}\sum_{k\in [K]}  \gamma^{di}_{jk}{v_{jk}} \nonumber
\end{align}
where  $\gamma^{ni}_{jk}$ and $\gamma^{di}_{jk}$ are the slopes of the approximate linear functions in $[L_j + (U_j-L_j)(k-1)/K;L_j + (U_j-L_j)(k)/K]$, $\forall k=1,\ldots, K$, computed as
\begin{align}
    \gamma^{ni}_{j,k+1} = \frac{K}{U_j-L_j}\left( n\left(L_j + \frac{(U_j-L_j)(k+1)}{K},\widehat{\bx}_i \right) -n\left(L_j + \frac{(U_j-L_j)(k)}{K},\widehat{\bx}_i \right)  \right),\; k=0,\ldots,K-1 \nonumber \\
       \gamma^{di}_{j,k+1} = \frac{K}{U_j-L_j}\left( d\left(L_j + \frac{(U_j-L_j)(k+1)}{K},\widehat{\bx}_i \right) -d\left(L_j + \frac{(U_j-L_j)(k)}{K},\widehat{\bx}_i \right)  \right),\; k=0,\ldots,K-1. \nonumber
\end{align}
We  can then approximate $F(\bz, \widehat{\bx}_i)$ as 
\[
F(\bz, \widehat{\bx}_i) \approx \frac{ \sum_{j}(n(L_j, \widehat{\bx}_i) + \frac{U_j-L_j}{K}\sum_{k\in [K]}  \gamma^{ni}_{jk}{v_{jk}})}{ \sum_{j} (d(L_j, \widehat{\bx}_i) + \frac{U_j-L_j}{K}\sum_{k\in [K]}  \gamma^{di}_{jk}{v_{jk}})}.
\]
The transformed/approximated problem will have the following parameters and variables
\begin{itemize}
    \item $\ba_i = \bigg[\gamma^{ni}_{jk} \Big|\; j \in [M], k \in [K]\bigg]$
    \item $\ba'_i = \bigg[\gamma^{di}_{jk} \Big|\; j \in [M], k \in [K]\bigg]$
    \item $b_i = \sum_{j \in [M]} n(L_j, \widehat{\bx}_i)$
    \item $b'_i = \sum_{j \in [M]} d(L_j, \widehat{\bx}_i)$
    \item $\bv \in \cV \myeq{}  \bigg\{v_{jk}\Big|\; v_{jk} \in \{0,1\}, v_{jk} \geq v_{j,k+1}, j \in [M], k \in [K] \bigg\}$.
\end{itemize}


 \section{Proof of Theorem~\ref{thm:thm1}}

\begin{theorem*} 
If the optimal decision when solving \ref{prob:DRO} is $\bz^{**}$ using $\bx^*_i$'s and $\widehat{\bz}^{**}$ using $\widehat{\bx}_i$'s, $F$ is $\tau$-Lipschitz in $\bx$, $X$ is bounded, and a scaled $\cL$ upper bounds $||\cdot||_2$ (i.e., $||\bx-\bx'||_2 \leq \max(k\cL(\bx,\bx'), \epsilon)$ for constants $k, \epsilon$) then, 
the following holds with probability $1-2\delta - 2 \delta_1$: $\bbE_{P^*}[ F(\widehat{\bz}^{**},\bx )] \geq \bbE_{P^*}[ F(\bz^{**},\bx )] - C/\sqrt{N} - (1 + 2\sqrt{\xi})\tau \epsilon - \epsilon_N - \epsilon_{N_T}$,
where $\epsilon_{K} = C_1 \cR_{K}(\cL \circ \cH) + C_2/\sqrt{K}$ and $\cR_{K}$ is the Rademacher complexity with $K$ samples and $C, C_1, C_2$ are constants dependent on $\delta, \delta_1, \xi, k, \tau$.
\end{theorem*}

 \begin{proof}
 We first list the mild assumptions: (1) $||\bx'-\bx||_2 \leq \max(k \cL(\bx',\bx), \epsilon)$ for some constant $k$ and a small constant $\epsilon$ and (2) space $X$ (that contains $\widehat{\bx}, \bx^*$) is bounded with a diameter $d_X$. The $k$ in the first assumption can be found since the space $X$ is bounded, and for close $\bx',\bx$, if needed, $\epsilon$ provides an upper bound. With this, it is easy to check that $$(1/N)\sum_i ||\bx^*_i-\widehat{\bx}_i||_2 \leq (1/N)\sum_i \max(k \cL(\bx^*_i,\widehat{\bx}_i), \epsilon) \leq \epsilon + (1/N)\sum_i k \cL(\bx^*_i,\widehat{\bx}_i) ).$$
 
We also have  $(1/N)\sum_i \cL(\widehat{\bx}_i, \bx^*_i) \leq \bbE [\cL_f] +  \epsilon_{N}$, where $\epsilon_{N}$ is of the form $C_1 \cR_{N}(\cL \circ \mathcal{F}) + \frac{C_2}{\sqrt{N}}$ for constants $C_1, C_2$ that depend on the probability term $\delta$, $\bbE[\cL_f]$ is the expected risk of function $f$, $\cR_{N}$ is the Rademacher complexity with $N$ samples, and $\cR_{N}(\cL \circ \mathcal{F})$ is well-defined for vector valued output of functions in $\mathcal{F}$ using each component of the output (see Proposition 1 in~\cite{reeve2020optimistic}). Next, the true risk of $f^*$ is assumed zero (text above the theorem in main paper): $\bbE[\cL_{f^*}] = 0$. Then, the ERM training using $N_T$ training data provides a high probability $1- \delta$ guarantee 
 which can be stated as $\bbE [\cL_f] \leq  \epsilon_{N_T}$, where $\epsilon_{N_T}$ is defined is same way as $\epsilon_N$ with $N_T$ replacing $N$.
 
 With this, we further have with probability $1 - 2 \delta$
 $$(1/N)\sum_i ||\bx^*_i-\widehat{\bx}_i||_2 \leq \epsilon + k(\epsilon_N + \epsilon_{N_T}).$$
 
 Let $\bz^{**}$ and $p^{**}_1, \ldots, p^{**}_N$ be the optimal solution found for $\bx^*_1, \ldots, \bx^*_N$. Note that $\bz^{**}$ and $p^{**}_1, \ldots, p^{**}_N$ is also a feasible point for the optimization with $\widehat{\bx}_1, \ldots, \widehat{\bx}_N$. We know that 
 \begin{align*}
|\sum_i  p^{**}_i F(\bz^{**}, \widehat{\bx}_i) - \sum_i  p^{**}_i F(\bz^{**}, \bx^*_i)| = |\sum_i  p^{**}_i (F(\bz^{**}, \widehat{\bx}_i) -  F(\bz^{**}, \bx^*_i))| \\
= |\sum_i  (p^{**}_i - 1/N)(F(\bz^{**}, \widehat{\bx}_i) -  F(\bz^{**}, \bx^*_i)) + \sum_i  (1/N)(F(\bz^{**}, \widehat{\bx}_i) -  F(\bz^{**}, \bx^*_i))| \end{align*}


We know that for any feasible $\bp, \bz$ 
 \begin{align*}
|\sum_i  p_i F(\bz, \widehat{\bx}_i) - \sum_i  p_i F(\bz, \bx^*_i)| = |\sum_i  p_i (F(\bz, \widehat{\bx}_i) -  F(\bz, \bx^*_i))| \\
= |\sum_i  (p_i - 1/N)(F(\bz, \widehat{\bx}_i) -  F(\bz, \bx^*_i)) + \sum_i  (1/N)(F(\bz, \widehat{\bx}_i) -  F(\bz, \bx^*_i))| \end{align*}
 
 Note that by Lispschitzness, 
 \begin{equation}
|\sum_i  (1/N)(F(\bz, \widehat{\bx}_i) -  F(\bz, \bx^*_i))| \leq \sum_i  (1/N)\tau|| \widehat{\bx}_i -   \bx^*_i||_2 \leq \tau (\epsilon +  k (\epsilon_N + \epsilon_{N_T})).     \label{eq:empricial}
 \end{equation}

 Also, $|\sum_i  (p_i - 1/N)(F(\bz, \widehat{\bx}_i) -  F(\bz, \bx^*_i))| \leq \sum_i |  (p_i - 1/N)(F(\bz, \widehat{\bx}_i) -  F(\bz, \bx^*_i))|$ and by Holder's inequality with $\infty, 1$ norm we get
 $$
  \sum_i |  (p_i - 1/N)(F(\bz, \widehat{\bx}_i) -  F(\bz, \bx^*_i))| \leq \Big( \max_i | (p_i - 1/N)| \Big ) \sum_i |F(\bz, \widehat{\bx}_i) -  F(\bz, \bx^*_i)|
 $$
 Since, $||\bp - \mathbf{1}/N||^2_2 \leq \xi/N^2$, thus, $\max_i | (p_i - 1/N)| \leq \sqrt{\xi}/N$. Hence, we get
 \begin{equation}
  \sum_i |  (p_i - 1/N)(F(\bz, \widehat{\bx}_i) -  F(\bz, \bx^*_i))| \leq \sqrt{\xi} (1/N)\sum_i |F(\bz, \widehat{\bx}_i) -  F(\bz, \bx^*_i)| \leq \sqrt{\xi}\tau (\epsilon + k (\epsilon_N + \epsilon_{N_T})) \label{eq:proofinter}
 \end{equation}
 With this, overall we get for any feasible $\bp, \bz$ 
 \begin{equation}
 |\sum_i  p_i F(\bz, \widehat{\bx}_i) - \sum_i  p_i F(\bz, \bx^*_i)| \leq (1 + \sqrt{\xi})\tau (\epsilon + k (\epsilon_N + \epsilon_{N_T})) = \psi     \label{eq:proofimp}
 \end{equation}
Note the following inequalities
\begin{align}
    \sum_i  p^{**}_i F(\bz^{**}, \bx^*_i) \leq & \sum_i \widehat{p}^{**}_i F(\bz^{**}, \bx^*_i)\\
    = & \Big(\sum_i \widehat{p}^{**}_i F(\bz^{**}, \bx^*_i) - \sum_i \widehat{p}^{**}_i F(\bz^{**}, \widehat{\bx}^*_i)\Big) + \sum_i \widehat{p}^{**}_i F(\bz^{**}, \widehat{\bx}^*_i)\\
    \leq & \psi + \sum_i \widehat{p}^{**}_i F(\bz^{**}, \widehat{\bx}^*_i) \label{eq:33}\\
    \leq & \psi + \sum_i \widehat{p}^{**}_i F(\widehat{\bz}^{**}, \widehat{\bx}^*_i) \label{eq:34}\\
    \leq & \psi + \sum_i p^{**}_i F(\widehat{\bz}^{**}, \widehat{\bx}^*_i) \label{eq:35}\\
    = & \psi + \Big(\sum_i p^{**}_i F(\widehat{\bz}^{**}, \widehat{\bx}^*_i) - \sum_i p^{**}_i F(\widehat{\bz}^{**}, \bx^*_i) \Big) + \sum_i p^{**}_i F(\widehat{\bz}^{**}, \bx^*_i) \\
    \leq & 2\psi + \sum_i p^{**}_i F(\widehat{\bz}^{**}, \bx^*_i)
\end{align}
where the first inequality is since $p^{**}_i$ is minimizer, Eq.~\ref{eq:33} is from Eq.~\ref{eq:proofimp},  Eq.~\ref{eq:34} is since $\widehat{\bz}^{**}$ is maximizer, Eq.~\ref{eq:35} is since $\widehat{p}^{**}_i$ is minimizer, and the last inequality is from Eq.~\ref{eq:proofimp}.

Next, by writing $p^{**}_i$ as $(p^{**}_i - 1/N) + 1/N$, we get from the above that
$$
 1/N \sum_i (F(\bz^{**}, \bx^*_i) \leq 2\psi + \sum_i (p^{**}_i - 1/N) (  F(\widehat{\bz}^{**}, \bx^*_i) - F(\bz^{**}, \bx^*_i)) +  1/N \sum_i F(\widehat{\bz}^{**}, \bx^*_i)
$$
By, Eq.~\ref{eq:proofinter} and that $\psi = (1 + \sqrt{\xi})\tau (\epsilon + k (\epsilon_N + \epsilon_{N_T}))$, we get
$$
 1/N \sum_i (F(\bz^{**}, \bx^*_i) \leq (1 + 2\sqrt{\xi})\tau (\epsilon + k (\epsilon_N + \epsilon_{N_T})) +  1/N \sum_i F(\widehat{\bz}^{**}, \bx^*_i)
$$
Also we absorb all constants in $(1 + 2\sqrt{\xi})\tau \epsilon$ to call it just $\epsilon$ and likewise, $(1 + 2\sqrt{\xi})\tau  k (\epsilon_N + \epsilon_{N_T})$ is just $\epsilon_N + \epsilon_{N_T}$.
Further, a standard concentration inequality for $\tau$-Lipschitz $F(\bz, \cdot )$ and bounded diameter $d_X$ of space $X$
can be invoked with the two decisions to get
\begin{align*}
    P\Bigg(\frac{1}{N} \sum_{i \in [N]} F(\bz^{**},\bx^*_i)   \geq \bbE_{\bx \sim P^*}[F(\bz^{**},\bx)] - t \Bigg) \geq  1 - \exp^{\frac{-2 N t^2}{\tau^2 d_X^2}} \\
        P\Bigg( \frac{1}{N} \sum_{i \in [N]} F(\widehat{\bz}^{**},\bx^*_i)  \leq t + \bbE_{\bx \sim P^*}[F(\widehat{\bz}^{**},\bx)] \Bigg) \geq  1 - \exp^{\frac{-2 N t^2}{\tau^2 d_X^2}} 
\end{align*}
Putting $\exp^{\frac{-2 N t^2}{\tau^2 d_X^2}}$ as $\delta_1$, we get $t$  of the form $C/\sqrt{N}$. Put all these together with a union bound yields, with probability $1 - 2 \delta - 2\delta_1$:
$$
\bbE_{\bx \sim P^*}[F(\bz^{**},\bx)] - C/\sqrt{N} - (1 + 2\sqrt{\xi})\tau \epsilon - \epsilon_N - \epsilon_{N_T} \leq \bbE_{\bx \sim P^*}[F(\widehat{\bz}^{**},\bx)]
$$
 \end{proof}

\section{Proof of Theorem~\ref{piecewiselinearerror}}\label{piecewiseproof}
\begin{theorem*}
For $F(\bz, \widehat{\bx}_i) = \frac{\sum_j n(z_j, \widehat{\bx}_i)}{\sum_j d(z_j, \widehat{\bx}_i)}$ as stated above and approximated as $\frac{\ba_i ^T \bv  + b_i}{\ba'^T_i \bv + b'_i}$, a PWLA approximation with $K$ pieces yields 
$\vert \cG(\bz^*) - \cG(\widehat{\bz}^{**}) \vert \leq O(1/K)$, where $\cG(\bz^*)$ and $\cG(\widehat{\bz}^{**})$ are the optimal objective values with approximation (MISOCP) and without the approximation respectively.
\end{theorem*}
\begin{proof}
The proof essentially follows by combining the results of the two lemmas below. We first prove the following two lemmas.

\begin{lemma*}
If 
$$
\Big\vert F(\bz, \widehat{\bx}_i) - \frac{\ba_i ^T \bv  + b_i}{\ba'^T_i \bv + b'_i} \Big\vert \leq {\epsilon_i}
$$
for some $\epsilon_i$ that is independent of $\widehat{\bz}$, then $\vert L^* - \widehat{L}^* \vert \leq {\max_i \{\epsilon_i\}}$, where $L^*$ and $\widehat{L}^*$ are the optimal objective values with and without the approximation.
\end{lemma*}
\begin{proof}
After the transformation, the decision variable $\bz$ changes from a continuous domain to $\bv$ in a discrete domain. Thus the original function $F_i(\bz) =  F(\bz, \widehat{\bx}_i) : \bz \longrightarrow \bbR$ and the approximate function $\widehat{F}_i(\bv) = \frac{\ba_i ^T \bv  + b_i}{\ba'^T_i \bv + b'_i} : \bv \longrightarrow \bbR$.
For ease of notation, given any $\bz$, let  $\bv = \cT(\bz)$ be the binary transformation of the continuous variables $\bz$,  and $\bz = \widetilde{\cT}(\bv)$ be the backward transformation from the binary variables $\bv$ to $\bz$.
From our assumption, we have $|F_i(\bz) - \widehat{F}_i(\cT(\bz))| \leq \epsilon_i$ for any $\bz$

Let us define (\textbf{OPT}) as the original optimization problem with continuous decision variable $\bz$ and (\textbf{Approx-OPT}) as the approximated problem with binary variable $\bv$. 
Let $q^*,\bl^*,\bz^*$ be an optimal solution to (\textbf{OPT}) and $q^{**},\bl^{**},\bv^{**}$ be an optimal solution to (\textbf{Approx-OPT}). 
Denote $\epsilon = \max_i \{\epsilon_i \}$, then we have  $|F_i(\bz) - \widehat{F}_i(\cT(\bz))| \leq \epsilon,\;  \forall i \in [N]$, for any $\bz$, which leads to 
(i) $F_i(\bz) \leq \widehat{F}_i(\cT(\bz)) + \epsilon$  and (ii) $\widehat{F}_i(\cT(\bz)) \leq F_i(\bz) + \epsilon,\; \forall i \in [N]$.
We also have (iii) $F_i(\widetilde{\cT}(\bv)) \leq \widehat{F}_i(\bv) + \epsilon$  and (iv) $\widehat{F}_i(\bv) \leq F_i(\widetilde{\cT}(\bv)) + \epsilon,\; \forall i \in [N]$.
We consider the following two cases:  $L^*\geq \widehat{L}^*$ or $L^*\le \widehat{L}^*$ as follows
\begin{itemize}
    \item If $L^*\geq \widehat{L}^*$, we first see that 
     \[q^* - l_i^* - F_i(\bz^*) = 0; \quad \forall i \in [N]\]
     From  Inequalities (i) and (ii) above, we will have 
     \[
     (q^*-\epsilon) - l_i^* - \widehat{F}_i(\cT(\bz^*)) \leq 0 \leq (q^*+\epsilon) - l_i^* - \widehat{F}_i(\cT(\bz^*)).
     \]
     Thus, there exists $\delta  \in [-\epsilon,\epsilon]$ such that $(q^*+\delta) - l_i^* - \widehat{F}_i(\cT(\bz^*))  = 0$, implying that $q^*+\delta,\bl^*,\cT(\bz^*)$ is feasible to (\textbf{Approx-OPT}), leading to $\widehat{L}^* \geq q^*+\delta - \sqrt{\rho\sum_i {(l^*)}_i^2}$. Thus, 
     \begin{align}
        |L^* - \widehat{L}^*| 
        &\leq \left| L^* - \left(q^*+\delta - \sqrt{\rho\sum_i {(l^*)}_i^2}\right)\right|\nonumber \\
        &= |\delta| \leq \epsilon.
     \end{align}
    \item If $L^*< \widehat{L}^*$, in analogy to the first case,  we also see that 
     \[q^{**} - l_i^{**} - F_i(\bv^{**}) = 0; \quad \forall i \in [N].\]
     From the above inequalities (iii) and (iv), it can also  be seen that there is $\delta \in [-\epsilon,\epsilon]$ such that $ (q^{**}+\delta) - l_i^{**} - \widehat{F}_i(\widehat{\cT}(\bv^{**}))  = 0$, implying that $q^{**}+\delta,\bl^{**},\widehat{\cT}(\bv^{**})$ is feasible to (\textbf{OPT}), leading to ${L}^* \geq q^{**}+\delta - \sqrt{\rho\sum_i {(l^{**})}_i^2}$. We then have the following inequalities 
     \begin{align}
                 |\widehat{L}^*- L^*| 
        &\leq \left| \widehat{L}^* - \left(q^{**}+\delta - \sqrt{\rho\sum_i {(l^{**})}_i^2}\right)\right|\nonumber \\
        &= |\delta| \leq \epsilon.
     \end{align}
\end{itemize}
Putting the two cases together, we have $|L^*-\widehat{L}^*|\leq \epsilon$, as desired.

\end{proof}
\begin{lemma*}
For $F(\bz, \widehat{\bx}_i) = \frac{\sum_j n(z_j, \widehat{\bx}_i)}{\sum_j d(z_j, \widehat{\bx}_i)}$, a PWLA approximation with $K$ pieces yields 
$$
\Big\vert F(\bz, \widehat{\bx}_i) - \frac{\ba_i ^T \bv  + b_i}{\ba'^T_i \bv + b'_i} \Big\vert \leq \frac{C}{K}
$$
with constant $C$ independent of ${\bz}$.
\end{lemma*} 

\begin{proof}
Let $C^n$, $C^d$ be the 
Lipschitz constant of $n(z_j, \widehat{\bx}_i)$ and    $d(z_j, \widehat{\bx}_i)$, respectively.
We use the  Lipschitz continuity of these functions  to get the following
\begin{align}
    |n(z_j, \widehat{\bx}_i) -  \widehat{n}(z_j, \widehat{\bx}_i)| &\leq \frac{U_j-L_j}{K} C^n \nonumber \\
    |d(z_j, \widehat{\bx}_i) -  \widehat{d}(z_j, \widehat{\bx}_i)| &\leq \frac{U_j-L_j}{K} C^d \nonumber
\end{align}
Then, by the above we  have 
\begin{align}
        \left|\sum_j n(z_j, \widehat{\bx}_i) - \sum_j \widehat{n}(z_j, \widehat{\bx}_i)\right| &\leq \frac{\sum_j U_j- \sum_jL_j}{K} C^n \myeq{} \epsilon^n \nonumber \\
                \left|\sum_j d(z_j, \widehat{\bx}_i) - \sum_j \widehat{d}(z_j, \widehat{\bx}_i)\right| &\leq \frac{\sum_j U_j- \sum_jL_j}{K} C^d \myeq{} \epsilon^d. \nonumber
\end{align}
Now, we write
\begin{align}
    |\widehat{F}_i(\bv) - F_i(\bz)| &=\left|\frac{\sum_j n(z_j, \widehat{\bx}_i)}{\sum_j d(z_j, \widehat{\bx}_i)} -  \frac{\sum_j \widehat{n}(z_j, \widehat{\bx}_i)}{\sum_j \widehat{d}(z_j, \widehat{\bx}_i)} \right| \nonumber\\
    &= \left|\frac{\sum_j n(z_j, \widehat{\bx}_i)\sum_j \widehat{d}(z_j, \widehat{\bx}_i) - \sum_j \widehat{n}(z_j, \widehat{\bx}_i)\sum_j {d}(z_j, \widehat{\bx}_i) }{\sum_j d(z_j, \widehat{\bx}_i)}  \right|,\nonumber
\end{align}
We handle the absolute value by considering the following two cases 
\begin{itemize}
    \item If $\sum_j n(z_j, \widehat{\bx}_i)\sum_j \widehat{d}(z_j, \widehat{\bx}_i) \geq \sum_j \widehat{n}(z_j, \widehat{\bx}_i)\sum_j {d}(z_j, \widehat{\bx}_i)$, then
    \begin{align}
    |\widehat{F}_i(\bv) - F_i(\bz)| &\leq  \left|\frac{\sum_j n(z_j, \widehat{\bx}_i)(\sum_j {d}(z_j, \widehat{\bx}_i) +\epsilon^d) - (\sum_j {n}(z_j, \widehat{\bx}_i)-\epsilon^n)\sum_j {d}(z_j, \widehat{\bx}_i) }{\sum_j d(z_j, \widehat{\bx}_i)}  \right|\nonumber \\
    &=  \left|\frac{\epsilon^d \sum_j n(z_j, \widehat{\bx}_i) +\epsilon^n\sum_j {d}(z_j, \widehat{\bx}_i) }{\sum_j d(z_j, \widehat{\bx}_i)}  \right|\nonumber\\
    &\le \max\{\epsilon^n,\epsilon^d\} \max_{\bz} \left\{\left|\frac{\sum_j n(z_j, \widehat{\bx}_i) +\sum_j {d}(z_j, \widehat{\bx}_i) }{\sum_j d(z_j, \widehat{\bx}_i)}  \right|\right\}\nonumber\\
    &=\frac{\sum_j U_j-\sum_j L_j}{K}\max\{C^n,C^d\} \max_{\bz} \left\{\left|\frac{\sum_j n(z_j, \widehat{\bx}_i) +\sum_j {d}(z_j, \widehat{\bx}_i) }{\sum_j d(z_j, \widehat{\bx}_i)}  \right|\right\}.\nonumber
\end{align}
\item   If $\sum_j n(z_j, \widehat{\bx}_i)\sum_j \widehat{d}(z_j, \widehat{\bx}_i) \leq \sum_j \widehat{n}(z_j, \widehat{\bx}_i)\sum_j {d}(z_j, \widehat{\bx}_i)$, similarly we have
  \begin{align}
    |\widehat{F}_i(\bv) - F_i(\bz)| &\leq  \left|\frac{ (\sum_j {n}(z_j, \widehat{\bx}_i)+\epsilon^n)\sum_j {d}(z_j, \widehat{\bx}_i) - \sum_j n(z_j, \widehat{\bx}_i)(\sum_j {d}(z_j, \widehat{\bx}_i) -\epsilon^d)  }{\sum_j d(z_j, \widehat{\bx}_i)}  \right|\nonumber \\
    &=  \left|\frac{\epsilon^d \sum_j n(z_j, \widehat{\bx}_i) +\epsilon^n\sum_j {d}(z_j, \widehat{\bx}_i) }{\sum_j d(z_j, \widehat{\bx}_i)}  \right|\nonumber\\
    &\le\frac{\sum_j U_j-\sum_j L_j}{K}\max\{C^n,C^d\} \max_{\bz} \left\{\left|\frac{\sum_j n(z_j, \widehat{\bx}_i) +\sum_j {d}(z_j, \widehat{\bx}_i) }{\sum_j d(z_j, \widehat{\bx}_i)}  \right|\right\}.\nonumber
    \end{align}
\end{itemize}
Therefore, if we let 
\[
C = {\left(\sum_j U_j-\sum_j L_j\right)}\max\{C^n,C^d\} \max_{\bz} \left\{\left|\frac{\sum_j n(z_j, \widehat{\bx}_i) +\sum_j {d}(z_j, \widehat{\bx}_i) }{\sum_j d(z_j, \widehat{\bx}_i)}  \right|\right\},
\]
which is independent of $\bz$, then we obtain the desired inequality $ |\widehat{F}_i(\bv) - F_i(\bz)|  \leq C/K$.
\end{proof}

 and 

Taking $\cG(\bz^*)$ as $L^*$, $\cG(\widehat{\bz}^{**})$ as $\widehat{L}^*$, and $\epsilon_i$ as $C/K$ we get the desired result for the theorem.
\end{proof}

\section{Proof of Lemma~\ref{lem:mean}}
\begin{lemma*}
We have
    $\left| \widehat{\text{Mean}}(F(\bz,\bx))  -  \widehat{\text{Mean}}^S(F(\bz,\bx)) \right| \leq \tau \epsilon$ and \nonumber 
    $ \left \vert \sqrt{\rho \widehat{\text{Var}}(F(\bz,\bx))} - \sqrt{\rho \widehat{\text{Var}}^S(F(\bz, \bx)) } \right \vert
    \leq 
    (\psi + \sqrt{2\tau\epsilon} ) \sqrt{\frac{2\tau \epsilon \xi}{N}}$.
\end{lemma*}
\begin{proof}

For the first result, By Lipschitzness,
\[
|F(\bz,\widehat{\bx}_i) - F(\bz,\bx^s)| \leq \tau \epsilon,\;\forall \bz ,   \forall \widehat{\bx}_i \mbox{ in cluster }s 
\]
The result follows by summing over $\widehat{\bx}^i$ and averaging.

{To get error bound for variance term},
let $I_s$ be the set of indices that belong to cluster $s$, thus, $\{I_s\}_{s \in [S]}$ is a partition of $[N]$ and $C_s = |I_s|$.
Let use define
\begin{align}
  \mu &= \frac{1}{N}\sum_{i \in [N]}  F(\bz,\widehat{\bx}_i) \nonumber\\
  \widehat{\mu} &= \frac{1}{N}\sum_s  C_s F(\bz,\widehat{\bx}^s) 
\end{align}
Let $\alpha_i = \mu - F(\bz, \widehat{\bx}_i)$ (or $F(\bz, \widehat{\bx}_i)= \mu + \alpha_i$), thus, $\sum_i \alpha_i = 0$.  As we know from Lipschitzness assumption that 
\begin{align}
|F(\bz,\widehat{\bx}_i) - F(\bz,\widehat{\bx}^s)| \leq \tau \epsilon,\;\forall \bz, i \in I_s, \label{eq:varproof1}
\end{align}
we always can write 
$F(\bz, \widehat{\bx}^s)$  as $F(\bz, \widehat{\bx}_i) + \beta_i = \mu + \alpha_i + \beta_i$ for any $\widehat{\bx}_i $ in cluster $s$, where $\beta_i$ are constants chosen such that 
\begin{align}
    -\tau\epsilon \leq \beta_i \leq \tau\epsilon,\\
    \frac{1}{N}\sum_{i \in [N]} \beta_i = \widehat{\mu} - \mu.
\end{align}
Then, we note that 
$$\sqrt{\sum_{i \in [N]} \left(\frac{1}{N} \sum_{i \in [N]}  F(\bz,\widehat{\bx}_i) - F(\bz,\widehat{\bx}_i)\right)^2} = \sqrt{\sum_{i \in [N]} \alpha_i^2}.$$
Also, we have
\begin{align*}
&\sqrt{\sum_s C_s  \left(\frac{1}{N} \sum_s C_s F(\bz,\widehat{\bx}^s) - F(\bz,\widehat{\bx}^s)\right)^2}= \sqrt{\sum_{i \in [N]} (\widehat{\mu} - \mu - \alpha_i - \beta_i )^2}, 
\end{align*}
as  $C_s$ is the number of points in cluster $s$ and  $\widehat{\mu} = \frac{1}{N}\sum_s  C_s F(\bz,\widehat{\bx}^s)$ and $F(\bz,\widehat{\bx}^s) = \mu + \alpha_i + \beta_i$ for all $i \in I_S$. Now, let us assume that
\begin{align*}
\sqrt{\rho\sum_s C_s  \left(\frac{1}{N} \sum_s C_s F(\bz,\widehat{\bx}^s) - F(\bz,\widehat{\bx}^s)\right)^2} \ge
\sqrt{\rho\sum_i \left(\frac{1}{N} \sum_i  F(\bz,\widehat{\bx}_i) - F(\bz,\widehat{\bx}_i)\right)^2},
\end{align*}
noting that the other case
\begin{align*}
\sqrt{\rho\sum_s C_s  \left(\frac{1}{N} \sum_s C_s F(\bz,\widehat{\bx}^s) - F(\bz,\widehat{\bx}^s)\right)^2} < \sqrt{\rho\sum_i \left(\frac{1}{N} \sum_i  F(\bz,\widehat{\bx}_i) - F(\bz,\widehat{\bx}_i)\right)^2}
\end{align*} 
can be handled similarly by  rewriting $F(\bz, \widehat{\bx}_i)$ as $\mu + \alpha_i + \beta_i$ and $F(\bz, \widehat{\bx}^s)$ as  $\mu + \alpha_i$.
We write
\begin{align*}
\sqrt{\sum_s C_s  \left(\frac{1}{N} \sum_s C_s F(\bz,\widehat{\bx}^s) - F(\bz,\widehat{\bx}^s)\right)^2} &= \sqrt{\sum_{i \in [N]} (\widehat{\mu} - \mu - \alpha_i - \beta_i )^2} \\
&= \sqrt{\sum_{i \in [N]} (\alpha_i + (\beta_i + \mu - \widehat{\mu}))^2} \\
&= \sqrt{\sum_{i \in [N]} \alpha_i^2 + 2\sum_{i \in [N]} \alpha_i (\beta_i + \mu - \widehat{\mu}) + \sum_{i \in [N]} (\beta_i + \mu - \widehat{\mu})^2}
\end{align*}
Using $\sum_i \alpha_i = 0$ and $|\mu- \widehat{\mu}| \leq \tau\epsilon, |\beta_i|\leq \tau\epsilon$ we get
\begin{align*}
\sqrt{\sum_s C_s  \left(\frac{1}{N} \sum_s C_s F(\bz,\widehat{\bx}^s) - F(\bz,\widehat{\bx}^s)\right)^2} &= \sqrt{\sum_{i \in [N]} \alpha_i^2 + 2\sum_{i \in [N]} \alpha_i \beta_i + \sum_{i \in [N]} (\beta_i + \mu - \widehat{\mu})^2} \\
 &\leq \sqrt{\sum_{i \in [N]} \alpha_i^2 + 2\sum_{i \in [N]} \alpha_i \beta_i + N(2\tau\epsilon)^2}. 
\end{align*}

Using $F_U, F_L$ are upper and lower limits of $F$ and let  $\psi = \sqrt(F_U - F_L)$, we further expand the inequalities  
\begin{align*}
\sqrt{\sum_s C_s  \left(\frac{1}{N} \sum_s C_s F(\bz,\widehat{\bx}^s) - F(\bz,\widehat{\bx}^s)\right)^2}
&\leq \sqrt{\sum_i \alpha_i^2 + 2N (F_U - F_L) \tau\epsilon +   N(2\tau\epsilon)^2} \\
&\stackrel{(a)}{\le} \sqrt{\sum_i \alpha_i^2} + \sqrt{2N (F_U - F_L) \tau\epsilon +   N(2\tau\epsilon)^2} \\
&\leq \sqrt{\sum_i \alpha_i^2} + (\psi + \sqrt{2\tau\epsilon} ) \sqrt{2N\tau\epsilon} \\
\end{align*}
where (a) is due to the fact that $\sqrt{a + b} \leq \sqrt{a} + \sqrt{b}$ for any non-negative numbers $a,b$.
Thus,  after plugging in $\rho = \frac{\xi}{N^2}$ we get
\begin{align*}
    & \Bigg \vert \sqrt{\rho\sum_{i \in [N]} \Big(\frac{1}{N} \sum_{i \in [N]}  F(\bz,\widehat{\bx}_i) - F(\bz,\widehat{\bx}_i)\Big)^2} - \sqrt{\rho\sum_s C_s  \left(\frac{1}{N} \sum_s C_s F(\bz,\widehat{\bx}^s) - F(\bz,\widehat{\bx}^s)\right)^2} \Bigg \vert \\
    &\leq  (\psi + \sqrt{2\tau\epsilon} ) \sqrt{\rho} \sqrt{2 N \tau \epsilon} \\
    &= (\psi + \sqrt{2\tau\epsilon} ) \sqrt{\frac{2\tau \epsilon \xi}{N}},
\end{align*}
as desired. 
\end{proof}

\section{Proof of Theorem~\ref{the:cluster}}
\begin{theorem*} 
Given the assumptions stated above,
and $\widehat{\bz}$ an optimal solution for  $\max_{\bz}\widehat{\cG}(\bz)$ and $\bz^*$ optimal for $\max_{\bz} \cG(\bz)$, the following holds:
\begin{align*}
 |\cG(\widehat{\bz}) - {\cG}(\bz^*)|\leq 2 (\tau \epsilon + \psi \sqrt{\frac{2\tau \epsilon \xi}{N}} + \frac{2\tau \epsilon \xi}{\sqrt{N}}).
\end{align*}
\end{theorem*} 
\begin{proof}
Let $\widehat{\bz}$ be an optimal solution to  $\max_{\bz}\widehat{\cG}(\bz)$ and $\bz^*$ be optimal to $\max_{\bz} \cG(\bz)$, we have 
\begin{align}
    |\cG(\widehat{\bz}) - {\cG}(\bz^*)|&\leq |\cG(\widehat{\bz}) - \widehat{\cG}(\widehat{\bz})| + |\widehat{\cG}(\widehat{\bz}) - {\cG}(\bz^*)| \nonumber 
\end{align}
The later can be further evaluated by considering two cases, $\widehat{\cG}(\widehat{\bz}) \geq {\cG}(\bz^*)$ and $\widehat{\cG}(\widehat{\bz}) < {\cG}(\bz^*)$. If $\widehat{\cG}(\widehat{\bz}) \geq {\cG}(\bz^*)$, then $|\widehat{\cG}(\widehat{\bz}) - {\cG}(\bz^*)| = \widehat{\cG}(\widehat{\bz}) - {\cG}(\bz^*) \leq \widehat{\cG}(\widehat{\bz}) - {\cG}(\widehat{\bz})$. The other case can be done similarly to have
\begin{align}
    |\cG(\widehat{\bz}) - {\cG}(\bz^*)|\leq 2 |\cG(\widehat{\bz}) - \widehat{\cG}(\widehat{\bz})\Big| &\leq 2|\widehat{\text{Mean}}^S(F(\widehat{\bz},\bx))  -  \widehat{\text{Mean}}(F(\widehat{\bz},\bx)) \Big |\nonumber \\
    &+ 2 \Big | \sqrt{\rho \widehat{\text{Var}}(F(\widehat{\bz},\bx))} - \sqrt{\rho \widehat{\text{Var}}^S(F(\widehat{\bz}, \bx)) } \Big |
    \nonumber 
\end{align}
Then using the two results in Lemma~\ref{lem:mean}, we get the required result.
\end{proof}

\section{Proof of Lemma \ref{lem:stratamean}}

\begin{lemma} \label{lem:stratamean} $\forall \bz$ with probability $\geq 1 - 2 \sum_t \exp^{\frac{-2 N_t\epsilon^2}{\tau^2 d_t^2}} $,
$\big \vert \widehat{\text{Mean}}(F(\bz,\bx))  -  \widehat{\text{Mean}}^T(F(\bz,\bx)) \big \vert \leq \epsilon   
$.
In other words,
\begin{align*}
&P\left( \left| \frac{1}{N}\sum_{j \in [M]}lF(\bz,\widehat{\bx}^j) - \frac{1}{N}\sum_{j \in [N]}[F(\bz,\bx^j)] \right| \leq \epsilon \right)  \geq \prod_t \big (1 - 2\exp^{\frac{-2 N_t\epsilon^2}{\tau^2 d_t^2}} \big)
\end{align*}
\end{lemma}

\begin{proof}
We utilize the concentration of Lipchitz functions. In particular, we have $\widehat{\bx}^1,\ldots,\widehat{\bx}^{N_t}$ which are sampled uniformly and independently from strata $t$ and bounded (has diameter $d_t$). Let $U_t$ denote the uniform probability distribution over the $C_t$ points in strata $t$. Let $I_t$ denote a set of indexes that lie in the strata $t$. Then, for our function $F(\bz,\bx)$ with Lipchitz constant $\tau$ we have : 
\begin{align*}
    P\Bigg(\Big \vert \frac{1}{N_t} \sum_{j \in [N_t]} F(\bz,\widehat{\bx}^j) - \bbE_{\bx \sim U_t}[F(\bz,\bx)] \Big \vert \leq \epsilon \Bigg) \geq  1 - 2\exp^{\frac{-2 N_t \epsilon^2}{\tau^2 d_t^2}} 
    \quad \forall t,\bz.
\end{align*}
Observe that by definition
$$
\bbE_{\bx \sim U_t}[F(\bz,\bx)] = \frac{1}{C_t}\sum_{j \in I_t}[F(\bz,\widehat{\bx}_j)].
$$
Hence,
\begin{align*}
 &P\Bigg(\Big \vert \frac{1}{N_t} \sum_{j \in [N_t]}F(\bz,\widehat{\bx}^j) - \frac{1}{C_t}\sum_{j \in I_t}[F(\bz,\widehat{\bx}_j)] \Big \vert \geq \epsilon \Bigg) \leq 2\exp^{\frac{-2 N_t \epsilon^2}{\tau^2 d_t^2}} \\
    &\Rightarrow P\Bigg(\Big \vert \sum_{j \in [N_t]}l_t F(\bz,\widehat{\bx}^j) - \sum_{j \in I_t}[F(\bz,\bx_j)] \Big \vert \geq C_t \epsilon \Bigg) \leq  2\exp^{\frac{-2 N_t \epsilon^2}{\tau^2 d_t^2}}.
\end{align*}

Call the event in the probability above as $E_t$. It is obvious that $E_t$ is independent over all different strata $t$'s due to the independent sampling of points across strata. Hence $\lnot E_t$ are also independent. Next, using product of independent events over all strata we get 
\begin{align*}
   &P\Big( \cap_t \lnot E_t \Big) \geq \prod_t \Big (1 -  2\exp^{\frac{-2 N_t \epsilon^2}{\tau^2 d_t^2}} \Big).
\end{align*}
Note that $\cap_t \lnot E_t$ implies $$\sum_t \Big \vert \sum_{j \in [N_t]}l_t F(\bz,\widehat{\bx}^j) - \sum_{j \in I_t}[F(\bz,\bx_j)] \Big \vert \leq \sum_t C_t \epsilon.$$
Noting that $|a + b |\leq |a| + |b|$ and the fact that $\{I_t\}_{t \in [T]}$ is a partition of $[N]$, the above implies that
$$\Big \vert \sum_t l_t \sum_{j \in [N_t]} F(\bz,\widehat{\bx}^j) - \sum_{j \in [N]}[F(\bz,\bx^j)] \Big \vert \leq \sum_t C_t \epsilon$$
This gives
\begin{align*}
    &P\Bigg(\Big \vert \sum_t l_t \sum_{j \in [N_t]} F(\bz,\widehat{\bx}^j) - \sum_{j \in [N]}[F(\bz,\bx^j)] \Big \vert \leq \sum_t C_t \epsilon \Bigg)  \geq \prod_t \Big (1 - 2\exp^{\frac{-2 N_t \epsilon^2}{\tau^2 d_t^2}} \Big ) \\ 
    & (\text{Then, since $N = \sum_t C_t $})\\
    &\Rightarrow P\Bigg(\Big \vert\frac{1}{N}\sum_t l_t \sum_{j \in [N_t]}F(\bz,\widehat{\bx}^j) - \frac{1}{N}\sum_{j \in [N]}[F(\bz,\bx_j)] \Big \vert \leq \epsilon \Bigg)  \geq \prod_t \Big ( 1 - 2\exp^{\frac{-2N_t \epsilon^2}{\tau^2 d_t^2}} \Big ) \\
    & (\text{Then, since $(1 -a)(1-b) \geq 1 - a - b)$})\\
    & \Rightarrow P\Bigg(\Big \vert\frac{1}{N}\sum_t l_t \sum_{j \in [N_t]}F(\bz,\widehat{\bx}^j) - \frac{1}{N}\sum_{j \in [N]}[F(\bz,\bx_j)] \Big \vert \leq \epsilon \Bigg)  \geq 1 - 2 \sum_t \exp^{\frac{-2N_t \epsilon^2}{\tau^2 d_t^2}},\\
\end{align*}
which is the desired inequality. 
\end{proof}

\section{Proof of Lemma~\ref{lem:stratavar}}

\begin{lemma} \label{lem:stratavar}
Define $D = \max_{\bz, \bx } |F(\bz, \bx)|$  for bounded function $F$. Then, $\forall \bz$ with probability  $\geq 1 - 4 \sum_t \exp^{\frac{-2 N_t \epsilon^2}{4 \tau^2 d_t^2 D^2 }}$, 
$
\left \vert \sqrt{\rho \widehat{\text{Var}}(F(\bz,\bx))} - \sqrt{\rho \widehat{\text{Var}}^T(F(\bz, \bx)) } \right \vert
    \leq \frac{2 \sqrt{\xi} \epsilon}{ \sqrt{\widehat{Var}(F(\bz,\bx))}}
$.
\end{lemma}

\begin{proof}
Fix $\bz$. Recall $I_t$ be the set of index that belong to strata $t$, thus, $\{I_t\}_{t \in [T]}$ is a partition of $[N]$ and $C_t = |I_t|$. For sake of simplicity, we use the shorthand for the sample/random variable $Y^j = F(\bz,\widehat{\bx}^j)$. Note that the samples are independent.
We use the following notations : 
$$\mu = \frac{1}{N}\sum_{i \in [N]}  F(\bz,\widehat{\bx}_i)$$
$$\widehat{\mu} = \frac{1}{N}\sum_t \sum_{j \in [N_t]} l_t Y^j $$
Note that $\sum_t \sum_{j \in [N_t]} l_t = N$.

The unnormalized weighted variance is
\begin{align*}
\widehat{Var}^T & = \sum_t   \sum_{j \in [N_t]} l_t \left(\widehat{\mu} - Y^j\right)^2 \\
& = \sum_t \sum_{j \in [N_t]} l_t \left( \widehat{\mu}^2 - 2 \widehat{\mu} Y^j + (Y^j)^2 \right) \\
& = N  \widehat{\mu}^2 - 2 \widehat{\mu} \sum_t \sum_{j \in [N_t]} l_t Y^j  + \sum_t \sum_{j \in [N_t]} l_t (Y^j)^2 \\
& = N  \widehat{\mu}^2 - 2 N  \widehat{\mu}^2  + \sum_t \sum_{j \in [N_t]} l_t (Y^j)^2 \\
& = \sum_t \sum_{j \in [N_t]} l_t (Y^j)^2 -  N  \widehat{\mu}^2
\end{align*}
We wish to compare this to 
$$
\widehat{Var} = \sum_{j \in [N]} F(\bz,\widehat{\bx}_j)^2 -  N  \mu^2
$$
Towards this end, we have
\begin{align}
|\widehat{Var}^T - \widehat{Var}| \leq | \sum_t \sum_{j \in [N_t]} l_t (Y^j)^2 - \sum_{j \in [N]} F(\bz,\widehat{\bx}_j)^2| + N | \widehat{\mu}^2 - \mu^2| \label{eq:varbound}
\end{align}

We know from Lipschitzness assumption that 
\begin{align}
|F(\bz,\widehat{\bx}_i) - F(\bz,\widehat{\bx}_j)| \leq \tau d_t,\;\forall \bz, i,j \in I_s 
\end{align}
Multiplying both sides by $|F(\bz,\widehat{\bx}_i) + F(\bz,\widehat{\bx}_j)|$ (which is $\leq 2D$), we get
\begin{align}
|F(\bz,\widehat{\bx}_i)^2 - F(\bz,\widehat{\bx}_j)^2| \leq 2 \tau d_t D,\;\forall \bz, i,j \in I_s \label{eq:squaredbound}
\end{align}

Let $U_t$ denote the uniform probability distribution over the $C_t$ points in strata $t$. Observe that by definition
$$
\bbE_{\bx \sim U_t}[(Y^j)^2] = \frac{1}{C_t}\sum_{j \in I_t}[F(\bz,\widehat{\bx}_j)^2]
$$
Then, by Hoeffding inequality and Equation~\ref{eq:squaredbound}
\begin{align*}
    P\Bigg(\Big \vert \frac{1}{N_t} \sum_{j \in [N_t]} (Y^j)^2 - \bbE_{\bx \sim U_t}[F(\bz,\bx)^2] \Big \vert \leq \epsilon \Bigg) \geq 1 - 2\exp^{\frac{-2 N_t \epsilon^2}{4 \tau^2 d_t^2 D^2}} 
    \quad \forall t,\bz
\end{align*}

Then, using the same sequence of steps as for Lemma~\ref{lem:stratamean}, we get 
\begin{align}
    P\Bigg(\Big \vert \frac{1}{N} \sum_t l_t \sum_{j \in [N_t]} (Y^j)^2 - \frac{1}{N} \sum_{j \in [N]} F(\bz,\bx_j)^2 \Big \vert \leq \epsilon \Bigg) \geq 1 - 2 \sum_t \exp^{\frac{-2 N_t \epsilon^2}{4 \tau^2 d_t^2 D^2}} 
    \quad \forall \bz \label{eq:squaredprob}
\end{align}

Also, we know from Lemma~\ref{lem:stratamean} that
\begin{align*}
    P\Bigg(\Big \vert \widehat{\mu} - \mu \Big \vert \leq \epsilon \Bigg) \geq 1 - 2 \sum_t \exp^{\frac{-2 N_t \epsilon^2}{\tau^2 d_t^2 }} 
    \quad \forall \bz
\end{align*}
Multiplying both sides of the term inside the probability by $|\widehat{\mu} + \mu|$ (which is $\leq 2 D$), we get
\begin{align*}
    P\Bigg(\Big \vert \widehat{\mu}^2 - \mu^2 \Big \vert \leq 2 \epsilon D \Bigg) \geq 1 - 2 \sum_t \exp^{\frac{-2 N_t \epsilon^2}{\tau^2 d_t^2 }} 
    \quad \forall \bz
\end{align*}

Replacing $2\epsilon D$ by $\epsilon$ (slight abuse of notation)
\begin{align}
    P\Bigg(\Big \vert \widehat{\mu}^2 - \mu^2 \Big \vert \leq  \epsilon \Bigg) \geq  1 - 2 \sum_t \exp^{\frac{-2 N_t \epsilon^2}{4 \tau^2 d_t^2 D^2 }} 
    \quad \forall \bz \label{eq:averageprob}
\end{align}
Denote the event in Equation~\ref{eq:squaredprob} as $A$ and Equation~\ref{eq:averageprob} as $B$, using union bound we get $P(\lnot A \lor \lnot B) \leq 4 \sum_t \exp^{\frac{-2 N_t \epsilon^2}{4 \tau^2 d_t^2 D^2 }}$, or by taking negation $P(A \land B) \geq 1 - 4 \sum_t \exp^{\frac{-2 N_t \epsilon^2}{4 \tau^2 d_t^2 D^2 }}$. $A \land B$ together with Equation~\ref{eq:varbound} implies that with probability  $1 - 4 \sum_t \exp^{\frac{-2 N_t \epsilon^2}{4 \tau^2 d_t^2 D^2 }}$
\begin{align}
|\widehat{Var}^T - \widehat{Var}| \leq  2N \epsilon \label{eq:varprobbound}
\end{align}
Then, note that
$$
\Big |\sqrt{\rho \widehat{Var}^T} - \sqrt{\rho \widehat{Var}} \Big | = \sqrt{\rho}\frac{|\widehat{Var}^T - \widehat{Var}|}{\sqrt{ \widehat{Var}^T} + \sqrt{\widehat{Var}}} \leq \frac{\sqrt{\xi}}{N} \frac{|\widehat{Var}^T - \widehat{Var}|}{ \sqrt{\widehat{Var}}}
$$
Then, using Equation~\ref{eq:varprobbound}, we get with probability  $1 - 4 \sum_t \exp^{\frac{-2 N_t \epsilon^2}{4 \tau^2 d_t^2 D^2 }}$
$$
\Big |\sqrt{\rho \widehat{Var}^T} - \sqrt{\rho \widehat{Var}} \Big | \leq  \frac{2 \sqrt{\xi} \epsilon}{ \sqrt{\widehat{Var}}}
$$
\end{proof}

\section{Proof of Theorem~\ref{the:strata}}
We prove a more general result stated below
\begin{theorem*}
Given the assumptions stated above,
and $\widehat{\bz}$ an optimal solution for  $\max_{\bz}\widehat{\cG}(\bz)$ and $\bz^*$ optimal for $\max_{\bz} \cG(\bz)$, the following statement holds with probability $\geq 1 - 2 \sum_t \exp^{\frac{-2 N_t\epsilon^2}{\tau^2 d_t^2}} - 4 \sum_t \exp^{\frac{-2 N_t \epsilon^2}{4 \tau^2 d_t^2 D^2 }}$:
\begin{align*}
|\cG(\widehat{\bz}) - {\cG}(\bz^*)|\leq 2  \epsilon \Bigg (1 + 2\sqrt{\frac{ \xi}{\widehat{Var}} } \Bigg ).
\end{align*}
For $N_* = \min_t {N_t}$, then the above can be written as
with probability $\geq 1 - 2 \sum_t \exp^{\frac{-2 \sqrt{N_*}\epsilon^2}{\tau^2 d_t^2}} - 4 \sum_t \exp^{\frac{-2 \sqrt{N_*} \epsilon^2}{4 \tau^2 d_t^2 D^2 }}$:
\begin{align*}
|\cG(\widehat{\bz}) - {\cG}(\bz^*)|\leq \frac{2  \epsilon}{(N_*)^{1/4}} \Bigg (1 + 2\sqrt{\frac{ \xi}{\widehat{Var}(F(\bz,\bx))} } \Bigg ).
\end{align*}
\end{theorem*}
\begin{proof}
Following style of proof of Theorem~\ref{the:cluster} using union bound with lemmas~\ref{lem:stratamean}
and~\ref{lem:stratavar} we get the first claim above. For the second claim set $\epsilon' = \sqrt{N_*}^{1/4} \epsilon$ and replace $\epsilon$ by $\epsilon'$ (note that $\sqrt{N_t} \geq \sqrt{N_*}$).
 \end{proof}


\section{Data Generation Details} \label{sec:datagen}

\textbf{(Synthetic) SSG: } Following standard terminology and set-up in SSG, for every target $j$, under a type specified by parameters $\bx$,
if the adversary attacks $j$ and the target is protected then the defender obtains reward $r_{\bx,j}^d$ and the adversary obtains $l_{\bx,j}^a$. Conversely, if the
defender is not protecting target $j$, then the defender obtains $l_{\bx,j}^d$ ($r_{\bx,j}^d > l_{\bx,j}^d$) and the adversary gets $r_{\bx,j}^a$ ($r_{\bx,j}^a > l_{\bx,j}^a$). Given
$z_j$ as  the marginal probability of defending target $j$, the expected utility of the defender and attacker of type $\bx$  for an attack on target $j$ is formulated as follows: $
u(z_j, \theta^d_{\bx}) = z_j r_{\bx,j}^d + (1-z_j) l_{\bx,j}^d$ and 
$h(z_j,\theta^a_{\bx}) = \lambda_{\bx}(z_j l_{\bx,j}^a + (1-z_j) r_{\bx,j}^a)$, where parameter
$\lambda_{\bx}\geq  0$ governs rationality. 
$\lambda_{\bx} \rightarrow 0$ means least rational, as the adversary chooses its attack uniformly at
random and $\lambda_{\bx}\rightarrow \infty$ means fully rational (i.e., attacks a target with highest utility). We compactly rewrite $u(z_j, \theta^d_{\bx}) = z_ja^d_{\bx,j} + l^d_{\bx,j}$ and $h(z_j, \theta^d_{\bx}) = -z_jc^a_{\bx,j} + l^a_{\bx,j}$. We add two layers of randomness to our \textit{parameters} $\{a^d_{\bx,j}, l^d_{\bx,j}, c^a_{\bx,j}, l^a_{\bx,j}|\forall j \in [M], \forall \bx\}$ by (1) generating i.i.d. samples from a mean shifted beta-distribution : $\text{low} + (\text{high} - \text{low})\textbf{Beta}(\alpha, \beta)$, and (2) then using these samples as means for the Gaussian distribution : $\cN(.,\sigma^2)$ to i.i.d. generate the final \textit{parameters}. In our experiments we chose : low = 5, high=8, $\alpha=3$, $\beta=3$, $\sigma^2 = 3$. 

\textbf{(Synthetic) Regressor for SSG utilities:} To validate Theorem \ref{thm:thm1}, we first fix a linear regressor $f^* = \langle s^*_{a^d_j}, b^*_{a^d_j}, s^*_{l^d_j}, b^*_{l^d_j}, s^*_{c^a_j}, b^*_{c^a_j}, s^*_{l^a_j}, b^*_{l^a_j} | \forall j \in [M] \rangle$ and sample $\{V^{*,a^d_{\bx}}, V^{*,l^d_{\bx}}, V^{*,c^a_{\bx}}, V^{*,l^a_{\bx}}|\forall \bx \in [N_T]\}$ to generate $\{a^{*,d_{\bx,j}}, l^{*,d_{\bx,j}}, c^{*,a_{\bx,j}}, l^{*,a_{\bx,j}}|\forall j \in [M], \forall \bx \in [N_T]\}$ such that $a^{*,d_{\bx,j}} = s^*_{a^d_j} *  V^{*,a^d_{\bx}} + b^*_{a^d_j}$, $l^{*,d_{\bx,j}} = s^*_{l^d_j} *  V^{*,l^d_{\bx}} + b^*_{l^d_j}$, $c^{*,a_{\bx,j}} = s^*_{c^a_j} *  V^{*,c^a_{\bx}} + b^*_{c^a_j}$, $l^{*,a_{\bx,j}} = s^*_{l^a_j} *  V^{*,l^a_{\bx}} + b^*_{l^a_j}$. Now a linear regressor $\widehat{f}$ is learnt on the given dataset of $N_T$ samples by minimizing the L-2 loss between outputs of $\widehat{f}$ :   $\{\widehat{a}^{d_{\bx,j}}, \widehat{l}^{d_{\bx,j}}, \widehat{c}^{a_{\bx,j}}, \widehat{l}^{a_{\bx,j}}|\forall j \in [M], \forall \bx \in [N_T]\}$ and actual utilities : $\{a^{*,d_{\bx,j}}, l^{*,d_{\bx,j}}, c^{*,a_{\bx,j}}, l^{*,a_{\bx,j}}|\forall j \in [M], \forall \bx \in [N_T]\}$. DRO is performed on both true and learnt utilities to get decisions and then evaluated on held out test set of true utilities. 


\textbf{(Semi-Synthetic) Maximum Capture Facility Cost Planning Problem (MC-FCP)}: The P\&R \cite{aros2013p} dataset provides fixed utilities for different facility locations which is useful when considering \textbf{MC-FCP}, where the utilities of each facility is a function of the budget allocated to it and our goal is to optimally distribute a limited budget across these facilities. Given the utilities of client $\bx$ : $V_{\bx, j} \forall j \in [M]$, we solve for parameters $\{a_{\bx,j}|j \in [M]\}$ governed by $V_{\bx, j} = a_{\bx, j} + b_{\bx}$, where $b_{\bx}$ is chosen as $\min_j V_{\bx,j}$, so that all $a_{\bx,j}$ are non negative, and utilities increase on allocating more budget. Once we have the parameters, we can write the utility function : $h(z_j, \theta_{\bx,j}) = a_{\bx,j} z_j + b_{\bx}$. Intuitively $b_{\bx}$ is the bias of the client $\bx$ and $a_{\bx,j} \geq 0$ is the rate at which the client's utility can be raised by allocating more budget to the $j^{th}$ facility.

The \textbf{MC-FLP} problem directly uses the utilities of client $\bx$ : $V_{\bx, j} \forall j \in [M]$ from the P\&R \citep{aros2013p} dataset, so \textbf{MC-FLP} is based completely on real data.

\section{Additional Results for Real Data } 
\label{sec:real_data}
\begin{table*}[t]
    \caption{Objective values of the baselines as a \% of the objective obtained by our approach across on \textbf{MC-FCP} across various settings.}
    \label{tab:FCP_baselines}
    \centering
\begin{tabular}{ccccccc}   
\toprule
    \multirow{3}*{$\xi$}& \multicolumn{3}{c}{\textbf{TTGA}} & \multicolumn{3}{c}{\textbf{PGA}}\\
    \cline{2-7}    
    & m=7 & m=10 & m=13 & m=7 & m=10 & m=13\\
    \cline{2-7}
1E2 & 51.6 & 50.3 & 61.2 & 38.7 & 45.7 & 55.3 \\ 
1E3 & 49.2 & 46.2 & 60.0 & 18.3 & 26.2 & 27.8 \\
1E4 & 48.2 & 45.0 & 30.4 & 15.0 & 18.2 & 19.1 \\  
\bottomrule
\end{tabular}
\end{table*}

\textbf{Baseline Performance on Real Data:}
Gradient based approaches failed to attain decent performance on this dataset on \textbf{MC-FCP} as the choice probabilities $F_i$ are near zero 
almost everywhere in the space of decisions $C$, and since the derivative of the objective w.r.t. the decision, ie. $\frac{\partial F_i}{\partial z} = F_i \times g_i(z)$, the baselines run into a vanishing gradient problem and fail to move from the initial point. This also demonstrates the advantage of an MISOCP solver which can locate good solutions despite the above issue. Nonetheless we use gradient clipping (clipped away from zero) to train our baselines on the dataset and the results are reported in Table \ref{tab:FCP_baselines}. 

\textbf{Training time (in secs) for our approach:} As demonstrated in Table \ref{tab:times}, even in the worst case our algorithm takes only about 15 minutes thus reflecting its scalability. 
 \begin{table*}[t]
    \caption{Training time (seconds) using our MISOCP formulation across various settings. }
    \label{tab:times}
    \centering
    \begin{tabular}{cccc|ccc}
    \toprule
    \multirow{3}*{$\xi$}& \multicolumn{3}{c}{\textbf{MC-FCP}} & \multicolumn{3}{c}{\textbf{MC-FLP}}\\ 
    \cline{2-7}
    & m=7 & m=10 & m=13 & m=10 & m=12 & m=14\\
    \cline{2-7}
         ERM & 62.16 & 182.47 & 128.54 & 11.62 & 28.20 & 11.98\\      
  1E2 & 271.51 & 267.50 & 80.24 & 11.57 & 33.42 & 30.92\\                                 
  1E3 & 80.94 & 297.64 & 900.84 & 11.66 & 33.07  & 33.56\\                                  
 1E4 & 263.53 & 558.82 & 820.54 & 32.84  & 33.76 & 42.28\\ 
\bottomrule
    \end{tabular}
\end{table*}
\section{On Choosing Optimal Number of Pieces} \label{sec:K}
We proved in Appendix~\ref{piecewiseproof} that piecewise linear approximation guarantees improve with increasing $K$. To choose a suitable K for our experiments, we varied the number of pieces ($K$) from 2 to 20 in steps of 2, and report the relevant statisitcs in Figure \ref{fig:choice_of_K}. We note that across various settings, the results have saturated by $K = 10$, and thus use $K = 10$ for all our experiments. 
\begin{figure}[htb]
    \centering
    \includegraphics[scale=0.5]{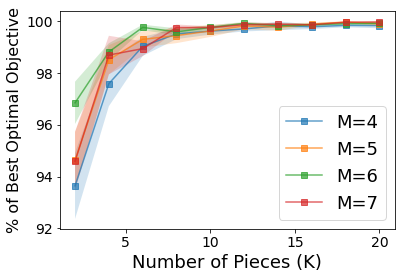}
    \caption{Optimal objective value achieved by varying number of pieces as a \% of the \textbf{Best OPT} - achieved at K=20. Results are shown for varying no. of alternatives $M$ and averaged over 10 generated \textbf{SSG} datasets with underlying parameters are $N=500, m=1, \xi=$1E6.}
    \label{fig:choice_of_K}
\end{figure}

\section{Converting Weighted Objective to MISOCP}
Let $$\widehat{\mu} = \frac{1}{N}\sum_t \sum_{j \in [N_t]} l_t F(\bz,\bx^j). $$ Further, let $Y^j = F(\bz,\widehat{\bx}^j)$.
Consider the stratified sampling objective
\begin{align}
\widehat{\mu} - \sqrt{\rho\sum_t l_t \sum_{j \in [N_t]}\left( \widehat{\mu} - Y^j\right)^2 }  \label{obj:clustering}
\end{align}
It is enough to show the conversion for the above as the clustering is a special case with $N_t = 1$ for all $t$.
As before we substitute $l_{t,j} = \frac{1}{N}\sum_t \sum_{j \in [N_t]} l_t Y^j - Y^j$ (notation $l$ is abused, but the constant $l_t$ subscript is $t$ and the variable subscript is $t,j$) for all $i  \in [N]$ and $q = \frac{1}{N}\sum_t \sum_{j \in [N_t]} l_t Y^j$.
Note that $\sum_{j \in N_t} l_{t,j} = \frac{N_t}{N}\sum_t l_t \sum_{j \in [N_t]}  Y^j - \sum_{j \in [N_t]}Y^j$, and since $l_t = \frac{C_t}{N_t}$, we have
$\sum_{j \in N_t} l_{t,j} = \frac{1}{N}\sum_t C_t \sum_{j \in [N_t]}  Y^j - \sum_{j \in [N_t]}Y^j$. Also, since $\sum_t C_t = N$ then 
$$
\sum_t C_t \sum_{j \in [N_t]} l_{t,j} =  \frac{\sum_t C_t}{N}\sum_t C_t \sum_{j \in [N_t]}  Y^j - \sum_t C_t \sum_{j \in [N_t]}Y^j = 0
$$
Also, $Y^j = F(\bz, \widehat{\bx}^j) = q - l_{t,j}$. The objective becomes $q - \sqrt{\rho \sum_t l_t \sum_{j \in [N_t]} l_{t,j}^2}$. Thus, like the original (non-clustered) problem the objective is concave, and the only non-convexity is in the constraint $F(\bz, \widehat{\bx}^j) = q - l_{t,j}$, which can be approximated as earlier.

For MISOCP, we move the part of the objective becomes the linear function $q - s$ with an additional constraint that \begin{align}
    \sqrt{\rho \sum_t l_t \sum_{j \in [N_t]} l_{t,j}^2} \leq s \label{eq:SOCP}
\end{align}
(Recall $\br$ is the vector of all variables).
The above is same as $\vert \vert  A \br \vert \vert_2 \leq \bc^T \br  $ for the constant matrix $A$ (with entries $0$ or $\sqrt{\rho l_t}$ at appropriate entries) and constant vector $\bc$ (with 1 in the $s$ component, rest 0's) that picks the $l_i$'s and $s$ respectively.

\end{document}